\documentclass[a4paper,fleqn]{cas-sc}

\usepackage{newclude}
\usepackage{amsmath}
\usepackage{cases}
\usepackage[square, comma, sort&compress, numbers]{natbib}
\usepackage{float} 
\usepackage{subfig}
\usepackage[ruled]{algorithm2e} 
\usepackage{diagbox}

\newtheorem{remark}{Remark}

\newtheorem{theorem}{Theorem}[section]

\newenvironment{proof}{{\noindent\it Proof.}\quad}{\hfill $\square$\par}

\def\d{\,\mathrm{d}}

\usepackage{multirow}

\usepackage[commandnameprefix=always]{changes}

\begin{document}
\let\WriteBookmarks\relax
\def\floatpagepagefraction{1}
\def\textpagefraction{.01}

\shorttitle{FENs}

\title [mode = title]{A Novel Fourier Feature Network for Solving Partial Differential Equations} 

\shortauthors{Qihong Yang et~al.}

\author[1]{Qihong Yang}[style=chinese, orcid=0000-0002-8398-7212]
\ead{yangqh0808@163.com}
\address[1]{organization={School of Mathematics, Sichuan University},
    city={Chengdu},
    postcode={610065},
    country={China}}

\author[1]{Zhijie Su}[style=chinese]
\ead{zhijiesu164@163.com}

\author[1]{Yangtao Deng}[style=chinese]
\ead{ytdeng1998@foxmail.com}

\author[1]{Qiaolin He}[style=chinese]
\cormark[1]
\ead{qlhejenny@scu.edu.cn}


\cortext[cor1]{Corresponding author}


\credit{Conceptualization of this study, Methodology, Writing - Original draft preparation}

\begin{abstract}
Building on the foundation of single-hidden-layer neural networks, Fourier Feature Networks (FENs) are proposed, which incorporate Fourier features using $\cos$, $\sin$, or a combination of both. Similar to Extreme Learning Machines (ELMs), FENs employ a single-hidden-layer architecture to generate a set of basis functions. The target function is then approximated as a linear combination of these basis functions, with the coefficients determined using the least squares method. However, unlike ELMs, which often rely on affine transformations to improve representational power, FENs can achieve high-precision solutions without requiring such transformations on the input variables. To evaluate the representational capacity of these networks, we search for an optimal scaling factor within a predefined range for the randomly initialized and fixed weights and biases. By adjusting this scaling factor, we ensure a fair comparison between FENs and ELMs using various activation functions, such as $\text{sigmoid}$, $\tanh$, and $\text{swish}$. Our numerical experiments demonstrate that FENs consistently achieve higher accuracy than ELMs.
\end{abstract}



\maketitle

\section{Introduction}
\label{sec:introduction}
In recent years, the application of neural networks in scientific computing has become increasingly widespread. Notably, the development of methods such as the Deep Ritz Method (DRM) \cite{yu2018deep}, the Deep Galerkin Method (DGM) \cite{sirignano2018dgm}, and Physics-Informed Neural Networks (PINNs) \cite{PINN} has garnered significant attention. A substantial body of work \cite{yuan2022pinn, lawal2022physics, hu2024physics2, huang2022applications, bararnia2022application} has demonstrated the tremendous potential of neural networks in this domain, particularly in the numerical solution of partial differential equations (PDEs).

Current neural network-based methods for solving PDEs can be broadly divided into two main categories. The first category is training-based, where the goal is to minimize the residual of the PDEs to fit the target function. This is typically achieved using gradient-based optimization algorithms such as Adam \cite{kingma2014adam} or L-BFGS \cite{liu1989limited}. Methods like the DRM \cite{yu2018deep}, the DGM \cite{sirignano2018dgm}, and the PINNs \cite{PINN} all fall under this category. A closely related approach is operator learning \cite{kovachki2024operator, kovachki2023neural, lu2021learning}, which aims to approximate mappings between infinite-dimensional Banach spaces using data. Although operator learning differs in formulation and application scope, it still fundamentally relies on training-based optimization. The second category comprises randomized neural network approaches, which do not involve iterative training. Instead, they generate a set of basis functions using a randomly initialized single-hidden-layer neural network. The target function is represented as a linear combination of these basis functions, and the PDE is discretized into a system of linear equations in terms of the coefficients of this combination. These coefficients are then obtained using the least squares method. Once the network is initialized, the weights and biases remain fixed; only the linear coefficients are optimized. This approach can be viewed either as a least-squares-based algorithm or as a randomized neural network method. Representative techniques include the Random Feature Method (RFM) \cite{chen2024optimization, chen2023random, RFM}, the Randomized Neural Network with Petrov–Galerkin methods (RNN-PG) \cite{shang2023randomized, shang2024randomized, shang2023randomized2, wang2024randomized}, and the Hidden-Layer Concatenated Extreme Learning Machine (HLConcELM) \cite{ni2023numerical}.

Despite extensive research on training-based neural networks for solving PDEs, 
these methods often ‌exhibit limited accuracy‌ and typically require ‌considerable computational time \cite{cuomo2022scientific}.
In contrast, randomized neural networks have gained significant attention in scientific computing due to their ability to solve PDEs both efficiently and accurately. However, this class of methods also presents several notable implementation challenges. \chadded{Extreme Learning Machine (ELM) is a single-layer feedforward neural network proposed in \cite{huang2006extreme}, which may not yield satisfactory results likely due to their reliance on Xavier or Kaiming initialization methods. Dong and Yang \cite{dong2022computing} presented a method for computing the optimal or near-optimal value of $R_m$ based on the differential evolution algorithm in ELM.} The RFM closely resembles Local Extreme Learning Machines (locELMs) \cite{dong2021local}, as both employ a domain decomposition strategy known as Partition of Unity (PoU). This technique enables the use of ELMs to solve PDEs through the strong form of the equations. While effective, domain decomposition significantly increases the complexity of the algorithm’s implementation.\chadded{The numerical experiments can not display good results in some complex cases.}
In the RNN-PG, the solution is formulated through the weak form of the PDEs. As a result, mesh generation and numerical integration are essential components of the algorithm, which further increase its implementation complexity. The HLConcELM differs from the above methods by modifying the architecture of ELMs, specifically, by adding an additional hidden layer and concatenating the outputs of the hidden layers. This architectural enhancement improves accuracy but also introduces additional computational overhead. In particular, when derivatives are computed using automatic differentiation \cite{baydin2018automatic}, complex network architectures can lead to significantly longer runtimes, especially when a large number of basis functions are involved. 

Although these algorithms can achieve significantly higher accuracy than training-based neural networks, they still fall short of achieving machine precision. To address this while maintaining low computational cost and avoiding increased network complexity, we focus on using ELMs for solving PDEs, deliberately steering clear of complexity-increasing operations such as domain decomposition and numerical integration. To further enhance the accuracy of ELMs without sacrificing simplicity, we propose the integration of Fourier features into the network architecture. Specifically, we introduce Fourier Feature Networks (FENs), which employ activation functions based on $\cos$, $\sin$, or a combination of both. In this work, we present and evaluate three types of FENs, each corresponding to one of these activation strategies, and conduct a detailed comparison of their accuracy against traditional ELMs activated by $\text{sigmoid}$, $\tanh$, and $\text{swish}$ functions. \chadded{The proposed FENs do not apply affine transformations to the input variables. Moreover, we use a uniform initialization method to initialize the weights and biases, and employ a method to search for the optimal scaling factor within a given range, which enables FENs to achieve optimal performance.}

This article is organized as follows. The network architecture and algorithmic details of ELMs, including the activation functions of $\text{sigmoid}$, $\tanh$, and $\text{swish}$ are introduced in Section \ref{sec:problems}. We propose Fourier Feature Networks, along with the three distinct activation modes of the FENs in Section \ref{sec:methods}. A series of numerical experiments designed to validate the efficacy of our methods are presented in Section \ref{sec:experiments}. The article concludes with Section \ref{sec:conclusions}, where we summarize our results, discuss the implications of our work, and suggest directions for future research.

\section{Preliminaries}
\label{sec:problems}

\subsection{The neural feature space}
Neural networks are widely recognized as nonlinear mapping functions that transform $d$-dimensional inputs into either lower or higher dimensional spaces.
 Architecturally, a neural network typically comprises three components: an input layer, multiple hidden layers, and an output layer. In this context, the neural feature space refers to the functional space represented by the outputs of the last hidden layer.

Assume that the input vector is $\boldsymbol{x} \in \mathbb{R}^d$. The nonlinear mapping represented by the hidden layers of the neural network is given by $\phi_{i}(\boldsymbol{x})$, where $1 \le i \le M$, and $M$ is the number of outputs of the last hidden layer.

Consequently, the neural feature space, denoted as $\mathcal{P}_{NN}$, is defined as a space spanned by the basis functions ${\phi_{i}}$, i.e.,
\begin{equation}
    \label{eq:basis}
    \mathcal{P}_{NN} = span\{\phi_1, \phi_2, \cdots, \phi_M\}.
\end{equation}

In neural networks, the output layer typically computes a linear combination of the outputs from the last hidden layer. Here, we assume that the output layer has a single output. Therefore, the function represented by the neural network can be expressed as
\begin{equation}
    \label{eq:u}
    u_M(\boldsymbol{x}) = \sum_{i=1}^{M} w_i \phi_i.
\end{equation}

\subsection{Extreme learning machine}
 Assume that the network has $d$ neurons in the input layer, corresponding to the input vector $\boldsymbol{x} \in \mathbb{R}^d$. The main idea of the ELM is to randomly initialize the weights $\boldsymbol{W} \in \mathbb{R}^{M \times d}$ and the biases $\boldsymbol{b} \in \mathbb{R}^{M}$ between the input and hidden layers. Then, the weights $\boldsymbol{w} \in \mathbb{R}^{M}$ from the hidden layer to the output layer are directly computed. This can be mathematically formulated as follows:
\begin{equation}
    \label{eq:ELM}
    u_M(\boldsymbol{x}) = \boldsymbol{w}^T \sigma(\boldsymbol{W}\boldsymbol{x}+\boldsymbol{b}),
\end{equation}
where $\sigma$ is the activation function, which acts elementwise on the vector $\boldsymbol{W}\boldsymbol{x} + \boldsymbol{b}$, and $u_M(\boldsymbol{x})$ is the function represented by the ELM. We only need to calculate $\boldsymbol{w}$.


\subsection{Affine transformation}
As we know, normalization plays a vital role in enhancing model performance and generalization capabilities in neural network training \cite{huang2023normalization}. Normalization accelerates the training process, enhances model generalization, prevents overfitting, and improves robustness to variations in weight initialization methods. In scientific computing, when neural networks are used for function approximation or solving PDEs, normalization serves as an essential preprocessing step to ensure numerical stability. In this context, affine transformations are commonly applied.

Let the input variables be $\boldsymbol{x} \in \Omega \subset \mathbb{R}^d$, where $\Omega$ is a closed set. The affine transformation maps the input variables $\boldsymbol{x}$ to a new vector $\tilde{\boldsymbol{x}} \in [-1, 1]^d \subset \mathbb{R}^d$, which is given by
\begin{equation}
    \label{eq:affine}
    T_i(x_i) = 2 \frac{x_i-x_i^{(l)}}{x_i^{(u)}-x_i^{(l)}} - 1,
\end{equation}
where $x_i$ is the $i$-th component of $\boldsymbol{x}$ and $x_i^{(l)}$ and $x_i^{(u)}$ are the lower and upper bounds of $x_i$, respectively.

When solving PDEs with ELM, applying affine transformations is critical to achieving superior performance. 
However, studies have demonstrated that omitting these transformations can lead to significantly degraded results. An ELM incorporating affine transformations can be expressed as follows
\begin{equation}
    \label{eq:affine_u}
    u_M(\boldsymbol{x}) = \boldsymbol{w}^T \sigma(\boldsymbol{W}\tilde{\boldsymbol{x}}+\boldsymbol{b}) = \boldsymbol{w}^T \sigma \left( \boldsymbol{W} \boldsymbol{T}(\boldsymbol{x}) +\boldsymbol{b} \right),
\end{equation}
where $\boldsymbol{T} = \left[ T_1(x_1), T_2(x_2), \cdots, T_d(x_d) \right]^T$. 

\section{Methodologies}
\label{sec:methods}
\subsection{Fourier Feature Networks}
It is observed that neural networks employing $\tanh$ activation function, while capable of generating acceptable results, fundamentally fail to attain machine-level precision.  
To address this limitation in approximation capacity and elevate solution accuracy, we introduce a novel architectural enhancement through Fourier feature embedding.

\subsubsection{Fourier Feature Network with a $\cos$ activation}
Gallant and White \cite{23903} pioneered the integration of Fourier features into neural network architectures through their development of the cosine squasher activation function.
Building upon this foundation, Silvescu \cite{silvescu1999fourier} implemented Fourier feature embedding via cosine-based activation operators. 
More recently, Ngom and Marin \cite{ngom2021fourier} introduced a Fourier Neural Network (FNN) with a single hidden layer activated by the cosine function. Although their experiments demonstrated promising results, the accuracy of the solutions still falls short of machine precision. Furthermore, their models omit the bias term, which can limit the expressive power of the neural network.

Inspired by these works, we propose a Fourier feature network with a $\cos$ activation. In this context, we also focus on a single-hidden-layer neural network. The target function can be represented as follows
\begin{equation}
    \label{eq:cosELM}
    u_M(\boldsymbol{x}) = \boldsymbol{w}^T \cos (\boldsymbol{W}\boldsymbol{x}+\boldsymbol{b}).
\end{equation}
It is noteworthy that, unlike networks activated by the $\text{sigmoid}$, $\tanh$, or $\text{swish}$ functions (as shown in Equation \eqref{eq:affine_u}), the input to the Fourier feature network with a cos activation does not require affine transformations.

\subsubsection{Fourier Feature Network with a $\sin$ activation}
Sitzmann et al. \cite{sitzmann2020implicit} introduced Sinusoidal Representation Networks (SIRENs), 
which utilize‌ the sine function as a periodic activation function. 
This design enables neural networks to accurately represent signals and their derivatives. Building on this work, Li et al. proposed an enhanced model called Spatially Collaged Coordinate Networks (SCONE), retaining‌ the sine activation function as a ‌foundational‌ component of the architecture. Belbute-Peres et al. \cite{de2022simple} also proposed an alternative improvement to SIRENs by incorporating a learnable scaling factor, 
enabling‌ automatic adjustment to inputs with ‌diverse‌ frequency characteristics.

Motivated by these studies, and following the natural progression from cosine to sine activation, we propose a Fourier feature network with a sin activation. As before, we focus on a single-hidden-layer neural network. The function to be learned is represented as follows
\begin{equation}
    \label{eq:sinELM}
    u_M(\boldsymbol{x}) = \boldsymbol{w}^T \sin (\boldsymbol{W}\boldsymbol{x}+\boldsymbol{b}).
\end{equation}
Similar to the Fourier feature network with a $\cos$ activation, the input to the network with a $\sin$ activation does not require affine transformations.

\subsubsection{Fourier Feature Network with $\cos$ and $\sin$ activations}
In 2013, Liu \cite{liu2013fourier} proposed a Fourier neural network activated by both cosine and sine functions. However, these networks were primarily applied to regression and classification tasks, and their accuracy was not satisfactory. More recent studies have incorporated Fourier features into neural networks by transforming input variables using cosine and sine functions, thereby embedding these features directly into the network’s input layer \cite{tancik2020fourier, wang2021eigenvector, li2023deep}. In line with Liu’s work, Fourier PINNs \cite{cooley2024fourier} also employ both cosine and sine activations as part of their basis functions. The crucial difference lies in how these basis functions are combined within the network to represent the learned function.

Inspired by these approaches, we propose a Fourier feature network utilizing both $\cos$ and $\sin$ activations. The output function is defined as
\begin{equation}
    \label{eq:cos_sinELM}
    u_M(\boldsymbol{x}) = (\boldsymbol{w}^{(1)})^T \cos (\boldsymbol{W}^{(1)} \boldsymbol{x} + \boldsymbol{b}^{(1)}) + (\boldsymbol{w}^{(2)})^T \sin (\boldsymbol{W}^{(2)}\boldsymbol{x} + \boldsymbol{b}^{(2)}),
\end{equation}
where $\boldsymbol{W}^{(1)} \in \mathbb{R}^{p \times d}$, $\boldsymbol{W}^{(2)} \in \mathbb{R}^{p \times d}$, $\boldsymbol{b}^{(1)} \in \mathbb{R}^{p}$, $\boldsymbol{b}^{(2)} \in \mathbb{R}^{p}$, $\boldsymbol{w}^{(1)} \in \mathbb{R}^p$, $\boldsymbol{w}^{(2)} \in \mathbb{R}^p$ and $M=2p$ is the number of basis functions.


\subsection{Function approximation}
Let us assume that a real-valued matrix $\mathbf{A}$ of size $N \times M$ is constructed from the basis functions $\Phi  = [\phi_1, \phi_2, \cdots, \phi_M]$, evaluated at discrete points within the dataset $S$. This matrix is defined by
\begin{equation}
    \label{eq:A_basis}
    \mathbf{A} = 
    \begin{bmatrix}
        \phi_1(\boldsymbol{x}_1) & \phi_2(\boldsymbol{x}_1) & \cdots & \phi_M(\boldsymbol{x}_1) \\
        \phi_1(\boldsymbol{x}_2) & \phi_2(\boldsymbol{x}_2) & \cdots & \phi_M(\boldsymbol{x}_2) \\
        \vdots & \vdots & \cdots & \vdots \\
        \phi_1(\boldsymbol{x}_N) & \phi_2(\boldsymbol{x}_N) & \cdots & \phi_M(\boldsymbol{x}_N) \\
    \end{bmatrix},
\end{equation}
where $\boldsymbol{x}_i \in S$ and $1 \leq i \leq N$. 

Based on Equation \eqref{eq:u}, the coefficient vector $\boldsymbol{w} = [w_1,w_2,..., w_M]^T$ can be determined by solving the linear system $\mathbf{A}  \boldsymbol{w} = \mathbf{F}$ using the least-squares method. Here, $\mathbf{A} = \Phi(S)$ denotes the basis function matrix formed by evaluating all basis functions at discrete data points in the dataset $S$, and $\mathbf{F} = [f(\boldsymbol{x}_1), f(\boldsymbol{x}_2), \ldots, f(\boldsymbol{x}_N)]^T$ corresponds to the vector of function values at these sampled points.

\subsection{Searching for optimal scaling factor}
The most widely used initialization methods in neural networks are Xavier initialization (also known as Glorot initialization) \cite{glorot2010understanding} and Kaiming initialization \cite{He_2015_ICCV}. These techniques are designed to initialize neural network weights according to specific distributions with carefully controlled variances. In conventional practice, both weight and bias parameters are typically assumed to originate from identical distributions. Consequently, biases are generally initialized using the same methodology as weights.

In this work, 
we employ a uniform distribution with unit variance to initialize both weight $\boldsymbol{W}$ and bias $\boldsymbol{b}$ parameters.
This choice is motivated by practical observations: neither Xavier nor Kaiming initialization consistently provides optimal results for our use cases. Crucially, the primary difference between these two methods lies in the scaling factor applied to weights and biases. 
 To address this limitation, we implement a systematic search protocol across a constrained scaling factor domain to identify initialization-sensitive optimal configurations.

Let the scaling factor be denoted by $\rho$. The functional mapping described by the neural network architecture in Equation \eqref{eq:ELM} can then be reparameterized in terms of $\rho$ as
\begin{equation}
    \label{eq:rho_ELM}
    u_M(\boldsymbol{x}) = \boldsymbol{w}^T \sigma(\rho (\boldsymbol{W}\boldsymbol{x}+\boldsymbol{b})).
\end{equation}
The protocol for locating the optimal scaling factor $\rho_{opt}$ is formally specified in Algorithm~\ref{algo:scale}. It is important to note that although the affine transformation is not explicitly shown in Equation \eqref{eq:rho_ELM}, it must still be applied when using $\text{sigmoid}$, $\tanh$, and $\text{swish}$ as activation functions.

\chadded{
It is important to note that the optimal scaling factor search method proposed in \cite{dong2022computing} used a differential evolution algorithm, whereas in our work, we enumerate all candidate scaling factors within a given range using a step-size approach. The scaling factor that minimizes the error is then selected from these candidates. This method is simple and easy to implement. Moreover, the work \cite{dong2022computing} focused exclusively on Gaussian activation functions, and the resulting numerical accuracy was not optimal in some cases. We argue that incorporating Fourier features is more essential ‌for achieving better performance.}

\begin{algorithm}[htp]
    \caption{Searching for optimal scaling factor}
    \label{algo:scale}
    Give the range of $\rho$ as $(\rho_{min}, \rho_{max}]$ and the step size $\rho_s$. \\
    Let $M$ be the number of basis functions and $\rho = \rho_{min} + \rho_s$. \\
    Randomly initialize the weights of the networks. \\
    \While{$\rho < \rho_{max}$}{
        Obtain the basis functions $\Phi$.\\
        Compute the matrix $\mathbf{A}$. \\
        Construct the right-hand vector $\mathbf{F}$.\\
        Express the current approximation of the solution as $u(\boldsymbol{x}) = \Phi \cdot \boldsymbol{w}$. \\
        Solve the linear system $\mathbf{A} \cdot \boldsymbol{w} = \mathbf{F}$ using least squares method to obtain the vector of coefficients $\boldsymbol{w}$. \\
        Let $Loss = \lVert \mathbf{A} \cdot \boldsymbol{w} - \mathbf{F} \rVert$ and record the $Loss$ and the current $\rho$ value. \\
        Let $\rho = \rho_{min} + \rho_s$. \\
    }
    The optimal scale factor $\rho_{opt}$ is the one with the smallest $Loss$.
\end{algorithm}

\begin{theorem}
    \label{thm:nn_app}
    (Theorem 2.2 \cite{shang2023randomized}) Given $p\ge 1, s , k , d \in N^+$, $s \ge k+1$. Let $\sigma$ be the logistic function or $\tanh$ function. Denote $\mathscr{F}_{s, p, d} := \{u \in \mathcal{W}^{s,p}([0,1]^d) : \lVert u \rVert _{\mathcal{W}^{s,p}([0,1]^d)}\le 1\}$, and $\mathscr{N}_\sigma(M_D,B_D) := \{u_M(x) = \boldsymbol{w}^T \sigma(\boldsymbol{W}\boldsymbol{x}+\boldsymbol{b}) \text{ with } N_D \leq M_D, \text{ and } \lvert w_{ij} \rvert \le B_D, \lvert b_{i} \rvert \le B_D\}$, where $N_D$ is the number of non-zero parameter elements in the hidden layer of the neural network. For any $\epsilon $ > 0 and u $\in \mathscr{F}_{s,p,d}$, there exists a neural network $u_M$ $\in \mathscr{N}_\sigma(M_D,B_D)$ with $M_D \le C \cdot \epsilon^{-d/s-k-\mu k}$, and $B_D \le C \cdot \epsilon^{-\theta}$ such that
    \begin{equation}
        \label{eq:nn_app}
    \parallel u-u_M\parallel_{\mathcal{W}^{s,p}([0,1]^d)} \le \epsilon,
\end{equation}
    where $C, \theta$ are constants depending on $d, s, p, k$; $\mu$ is an arbitrarily small positive number.
\end{theorem}

In Theorem~\ref{thm:nn_app}, it is established that for any desired approximation accuracy $\epsilon$, there exist suitable neural network parameters that allow the constructed network to approximate the target function within this tolerance. However, the challenge lies in the fact that these optimal parameters are generally unknown in practice. In the RNN-PG framework \cite{shang2023randomized}, these parameters are typically obtained through repeated random initialization, which are both computationally inefficient and potentially suboptimal.

Theorem~\ref{thm:opt_func_app} demonstrates that, for any given initialization, it is possible to improve the applicability of the network parameters by appropriately adjusting the scaling factor $\rho$.

\begin{theorem}
    \label{thm:opt_func_app}
    \textbf{(Existence of the optimal scaling factor)}
    Assume that $u(\boldsymbol{x})$ is a continuous function and $\sigma$ is a continuous activation function. Let $u_M(\boldsymbol{x}) = \boldsymbol{w}^T(\rho) \sigma(\rho (\boldsymbol{W}\boldsymbol{x}+\boldsymbol{b}))$ and $u_M$ is the projection of $u$ onto the linear space spanned by the basis functions $\sigma(\rho (\boldsymbol{W}\boldsymbol{x}+\boldsymbol{b}))$, where $\boldsymbol{w}(\rho)$ is a vector consisting of a set of functions related to $\rho$. Then, $\forall \rho_{max}>0$, $\exists \rho_{opt} \in [0, \rho_{max}]$ such that
    \begin{equation*}
        \lVert u - u_M(\boldsymbol{x}; \rho_{opt}) \rVert = \min_{0 \leq \rho \leq \rho_{max}} \lVert u(\boldsymbol{x}) - u_M(\boldsymbol{x}; \rho) \rVert.
    \end{equation*}
\end{theorem}

\begin{proof}
    Let $\Phi$ be the set of basis functions $\sigma(\rho (\boldsymbol{W}\boldsymbol{x}+\boldsymbol{b}))$, and $S=[\boldsymbol{x}_1, \boldsymbol{x}_2, \cdots, \boldsymbol{x}_N]^T$ be the set of collocation points in the domain $\Omega$. According to Equation \eqref{eq:A_basis}, we substitute the set of collocation points $S$ into it to obtain 
    \begin{equation*}
        \mathbf{A}(\rho)\boldsymbol{w}(\rho)=\mathbf{F},
    \end{equation*}
    where $\mathbf{F} = [u(\boldsymbol{x}_1), u(\boldsymbol{x}_2), \cdots, u(\boldsymbol{x}_N)]^T$.

The continuity of the activation function $\sigma$ implies continuity of $\mathbf{A} (\rho)$ in $\rho$. Given $\mathbf{F}$ is likewise continuous in $\rho$, it follows that both $\boldsymbol{w}(\rho)$ and ultimately $u_M(\boldsymbol{x};\rho)$ are continuous in $\rho$. Let $e(\boldsymbol{x}; \rho)=u(\boldsymbol{x}) - u_M(\boldsymbol{x}; \rho)$, where $e(\boldsymbol{x}; \rho)$ is continuous in $\rho$. By the continuity of the norm,  $\lVert e(\boldsymbol{x}; \rho) \rVert$ is likewise continuous in $\rho$.  Consequently,  $\forall \rho_{max} > 0$, $\lVert e(\boldsymbol{x}; \rho) \rVert$ attains its maximum and minimum values on the closed interval $[0,\rho_{max}]$.
\end{proof}

\begin{remark}
In practice, the search for the optimal scaling factor typically excludes zero as an initial candidate, since empirical evidence indicates that $\rho =0 $ is rarely an optimal choice.
\end{remark}

Figure~\ref{fig:2D_scale} shows how the $L_{\infty}$ error evolves with the scaling factor across different basis function counts. The ELM with the $\tanh$ activation function is employed to solve the two-dimensional Equation~\eqref{eq:2D_function}. The scaling factor $\rho$ is varied within the range $(0,10]$ with a step size of $0.01$. As shown in the figure, increasing the number of basis functions enhances the representational capacity of the ELM, leading to a significant reduction in error. Furthermore, 
while the error fluctuates as the scaling factor $\rho$ changes, reveals the existence of an optimal scaling factor (denoted as $\rho_{opt}$) that minimizes the error.
This observation numerically validates Theorem~\ref{thm:opt_func_app}.

\chadded{
It should be noted that although Figure~\ref{fig:2D_scale} displays only the error curves of the ELM with $\tanh$ activation function under varying scaling factor $\rho$, a similar trend occurs in FENs. Moreover, FENs typically operate across a broader $\rho$ search range and achieve lower errors.
}

\begin{figure}[htbp]
    \begin{minipage}{0.9\linewidth}
        \centering
        \includegraphics[width=0.8\textwidth]{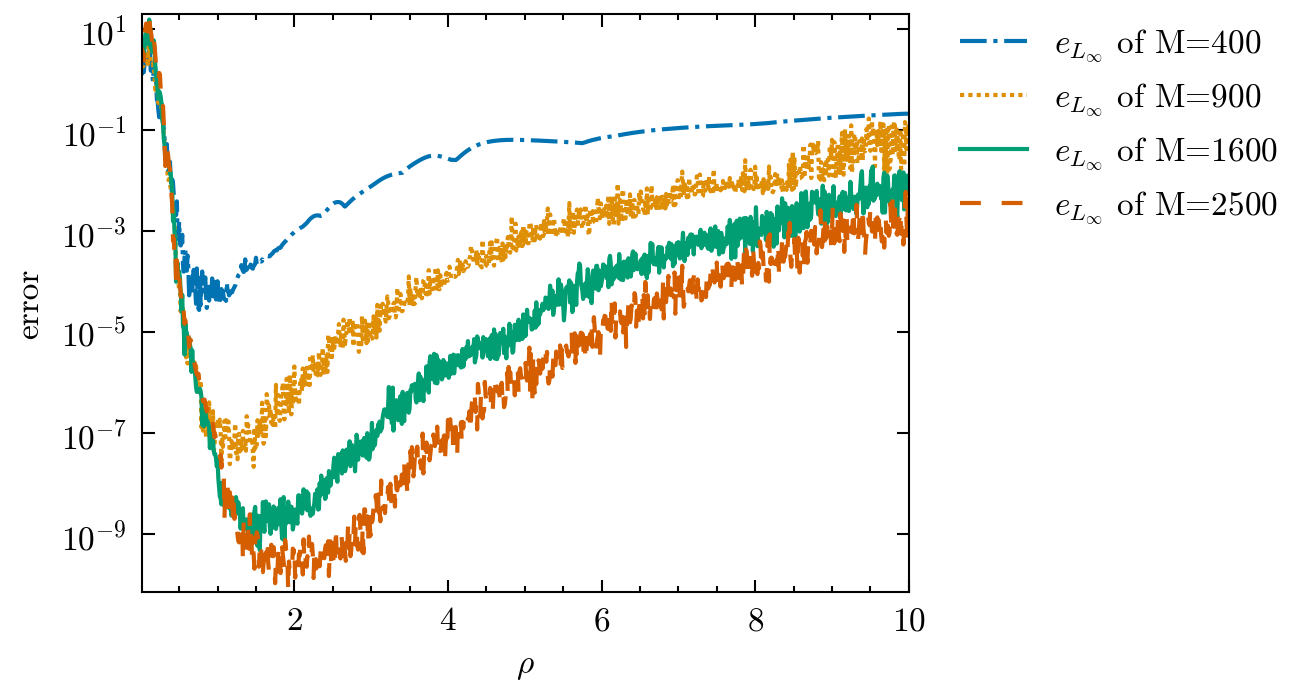}
    \end{minipage}
    \caption{Error variation with the scaling factor $\rho$ for different numbers of basis functions. 
    }
    \label{fig:2D_scale}
\end{figure}

\subsection{Solving linear PDEs}
Consider the following generic linear PDE
\begin{numcases}{}
    \mathcal{L}u= f,    \  \mbox{in} \enspace \Omega,      \label{eq:control_equation}    \\
    \mathcal{B}u = g,   \  \mbox{on} \enspace \partial \Omega,  \label{eq:boundary_condition}
 \end{numcases}
where $u$ is the scalar field function to be approximated, $\mathcal{L}$ is a linear PDE operator defined within the domain $\Omega$ and $\mathcal{B}$ is a linear boundary operator acting on the boundary $\partial \Omega$. The functions $f$ and $g$ represent the source term and the boundary condition, respectively. 

Let $\Phi$ denote the basis functions generated by the outputs of the FENs. Applying the operators $\mathcal{L}$ and $\mathcal{B}$ to these basis functions yields
\begin{numcases}{}
    \mathcal{L}\Phi = (\mathcal{L}\phi_1, \mathcal{L}\phi_2, \cdots, \mathcal{L}\phi_M),      \label{eq:L_ctrl_eq}    \\
    \mathcal{B}\Phi = (\mathcal{B}\phi_1, \mathcal{B}\phi_2, \cdots, \mathcal{B}\phi_M).      \label{eq:B_boun_con}
\end{numcases}
Assume that the collocation dataset $S$ consists of $N = N_r + N_b$ points, with $N_r$ points distributed within the domain $\Omega$ and $N_b$ points allocated on the boundary $\partial \Omega$. These points are partitioned into two subsets: $S_r$ (interior points) and $S_b$ (boundary points). Under this configuration, the matrix $\mathbf{A}$, which corresponds to the left-hand side of the PDE system, is constructed as follows
\begin{equation}
    \label{eq:linear_system_left}
    \begin{aligned}
    \mathbf{A} &= 
    \begin{bmatrix}
        \mathcal{L}\Phi(S_r) \\
        \mathcal{B}\Phi(S_b)
    \end{bmatrix} \\
    &= 
    \begin{bmatrix}
        \mathcal{L}\phi_1(\boldsymbol{x}_1) & \mathcal{L}\phi_2(\boldsymbol{x}_1) & \cdots & \mathcal{L}\phi_M(\boldsymbol{x}_1) \\
        \vdots & \vdots & \cdots  & \vdots \\
        \mathcal{L}\phi_1(\boldsymbol{x}_{N_r}) & \mathcal{L}\phi_2(\boldsymbol{x}_{N_r}) & \cdots & \mathcal{L}\phi_M(\boldsymbol{x}_{N_r}) \\
        \mathcal{B}\phi_1(\boldsymbol{x}_{N_r+1}) & \mathcal{B}\phi_2(\boldsymbol{x}_{N_r+1}) & \cdots & \mathcal{B}\phi_M(\boldsymbol{x}_{N_r+1}) \\
        \vdots & \vdots & \cdots  & \vdots \\
        \mathcal{B}\phi_1(\boldsymbol{x}_{N_r+N_b}) & \mathcal{B}\phi_2(\boldsymbol{x}_{N_r+N_b}) & \cdots & \mathcal{B}\phi_M(\boldsymbol{x}_{N_r+N_b}) \\
    \end{bmatrix}.
    \end{aligned}
\end{equation}

As described in the previous section, the coefficient vector $\boldsymbol{w}$ is obtained by solving the linear system $\mathbf{A} \boldsymbol{w} = \mathbf{F}$, where the right-hand side vector $\mathbf{F}$ is given by
\begin{equation}
    \label{eq:linear_system_right}
    \boldsymbol{F} = [ f(\boldsymbol{x}_1), \cdots, f(\boldsymbol{x}_{N_r}), g(\boldsymbol{x}_{N_r +1}), \cdots, g(\boldsymbol{x}_{N_r + N_b}) ]^T.
\end{equation}
Here, $f(\boldsymbol{x}_i)$ for $1 \le i \le N_r$ corresponds to the source term evaluated at the ‌interior collocation points, whereas $g(\boldsymbol{x}_i)$ for $N_r+1 \le i \le N_r+N_b$ represents the boundary conditions imposed at the ‌boundary collocation points. \chadded{It is worth noting that, for time-dependent problems, we adopt a unified treatment of both initial and boundary conditions by handling them as boundary constraints. This approach simplifies the overall formulation. Additionally, we do not explicitly assign different weights to the governing equations and the initial or boundary conditions; instead, all constraints are treated with equal importance, with an implicit weight of 1. The training and test sets are chosen to be identical, as our objective is to compute solution values on the given dataset. This setup aligns with the traditional goal in computational mathematics, which is to approximate the solution at prescribed or unknown locations.} 

\subsection{Approxmation theory of FENs}

Let us consider functions defined on $\mathbb{R}^d$ that admit the following Fourier integral representation
\begin{equation}
    \label{eq:Fourier_representation}
    u(\boldsymbol{x}) = \int_{\mathbb{R}^d} e^{\mathbb{i} \boldsymbol{W}_r\boldsymbol{x}} \tilde{F}(\d{\boldsymbol{W}_r}),
\end{equation}
where $\tilde{F}(\d{\boldsymbol{W}_r})=e^{i\theta(\boldsymbol{W}_r)}F(\d{\boldsymbol{W}_r})$ is a unique complex-valued measure (referred to as the Fourier distribution), with $F(\d{\boldsymbol{W}_r})$ representing the magnitude distribution, $\theta(\boldsymbol{W}_r)$ denoting the phase, and $\boldsymbol{W}_r \in \mathbb{R}^{1 \times d}$.
For each $C>0$, we define
\begin{equation}
    \label{eq:F_C_Omega}
    \mathbb{F}_{C,\Omega}= \left\{ u:\Omega \to \mathbb{R} \big| C_{u, \Omega} = \int_{\mathbb{R}^d} F(\d{\boldsymbol{W}_r}) \le C \right\}.
\end{equation}


\begin{theorem}
    \label{thm:cos}
    \textbf{(Approximation error of FEN with a $\cos$ activation \cite{zhumekenov2021approximation})} Let $u_M = \sum_{i=1}^M w_i \cos(\boldsymbol{W}_i \boldsymbol{x} + b_i)$ for $\boldsymbol{W}_i \in \mathbb{R}^{1 \times d}$, $b_i \in \mathbb{R}$, $\lvert w_i \rvert \le \frac{C}{M}$, and $\boldsymbol{x} \in \mathbb{R}^d$. For each function $u \in \mathbb{F}_{C, \Omega}$ and any probability measure $\mu$, there exists $u_M$ ($M \ge 1$), such that     
    \begin{equation}
        \int_{\Omega} \lvert u(\boldsymbol{x}) - u_M(\boldsymbol{x}) \rvert ^2 \mu(\d \boldsymbol{x}) \le \frac{C^2}{M}. \nonumber
    \end{equation}
\end{theorem}

\begin{theorem}
    \label{thm:sin}
    \textbf{(Approximation error of FEN with a $\sin$ activation)}
    Let $u_M = \sum_{i=1}^M w_i \sin(\boldsymbol{W}_i \boldsymbol{x} + b_i)$ for $\boldsymbol{W}_i \in \mathbb{R}^{1 \times d}$, $b_i \in \mathbb{R}$, $\lvert w_i \rvert \le \frac{C}{M}$ and $\boldsymbol{x} \in \mathbb{R}^d$. For each function $u \in \mathbb{F}_{C, \Omega}$ and any probability measure $\mu$, there exists $u_M$ ($M \ge 1$), such that  
    \begin{equation}
        \int_{\Omega} \lvert u(\boldsymbol{x}) - u_M(\boldsymbol{x}) \rvert ^2 \mu(\d \boldsymbol{x}) \le \frac{C^2}{M}. \nonumber
    \end{equation}
\end{theorem}
According to Theorem  \ref{thm:cos}, the proof of Theorem \ref{thm:sin} is straightforward. 
\begin{theorem}
    \label{thm:cos_sin}
    \textbf{(Approxmation error of FEN with $\cos$ and $\sin$ activations)} Let 
    \begin{equation}
        \label{eq:u_M_cos_sin}
        u_M = \sum_{i=1}^p w_{i}^{(1)} \cos(\boldsymbol{W}_i^{(1)} \boldsymbol{x} + b_i^{(1)}) + \sum_{i=1}^p w_{i}^{(2)} \sin(\boldsymbol{W}_i^{(2)} \boldsymbol{x} + b_i^{(2)}) \nonumber
    \end{equation}
    for $\boldsymbol{W}_i^{(1)} \in \mathbb{R}^{1 \times d}$, $\boldsymbol{W}_i^{(2)} \in \mathbb{R}^{1 \times d}$, $b_i^{(1)} \in \mathbb{R}$, $b_i^{(2)} \in \mathbb{R}$, $\lvert w_i^{(1)} \rvert \le \frac{C}{p}$, $\lvert w_i^{(2)} \rvert  \le \frac{C}{p}$, and $\boldsymbol{x} \in \mathbb{R}^d$. For each function $u \in \mathbb{F}_{C, \Omega}$ and any probability measure $\mu$, there exists $u_M$ ($M =2p \ge 2$), such that  
    \begin{equation}
        \int_{\Omega} \lvert u(\boldsymbol{x}) - u_M(\boldsymbol{x}) \rvert ^2 \mu(\d \boldsymbol{x}) \le \frac{4C^2}{p}. \nonumber
    \end{equation}
\end{theorem}

\begin{proof}
    From Theorem \ref{thm:cos}, for any given function $u \in \mathbb{F}_{C, \Omega}$, there is a sum $u_p^{(1)} = \sum_{i=1}^p w_i^{(1)} \cos(\boldsymbol{W}_i^{(1)} \boldsymbol{x} + b_i^{(1)})$, such that
    \begin{equation}
        \int_{\Omega} \lvert \frac{1}{2}u(\boldsymbol{x}) - u_p^{(1)}(\boldsymbol{x}) \rvert ^2 \mu(\d \boldsymbol{x}) \le \frac{C^2}{p}. \nonumber
    \end{equation}
 From Theorem \ref{thm:sin}, for any given function $u \in \mathbb{F}_{C, \Omega}$, there is a sum $u_p^{(2)} = \sum_{i=1}^p w_i^{(2)} \cos(\boldsymbol{W}_i^{(2)} \boldsymbol{x} + b_i^{(2)})$, such that
    \begin{equation}
        \int_{\Omega} \lvert \frac{1}{2}u(\boldsymbol{x}) - u_p^{(2)}(\boldsymbol{x}) \rvert ^2 \mu(\d \boldsymbol{x}) \le \frac{C^2}{p}. \nonumber
    \end{equation}

    Since $u_M = u_p^{(1)} + u_p^{(2)}$, then
    \begin{equation}
        \begin{aligned}
        \int_{\Omega} \lvert u(\boldsymbol{x}) - u_M(\boldsymbol{x}) \rvert ^2 \mu(\d \boldsymbol{x}) &= \int_{\Omega} \lvert \frac{1}{2}u(\boldsymbol{x}) - u_p^{(1)}(\boldsymbol{x}) + \frac{1}{2}u(\boldsymbol{x}) - u_p^{(2)}(\boldsymbol{x}) \rvert ^2 \mu(\d \boldsymbol{x}) \\
        & \le \int_{\Omega} 2\lvert \frac{1}{2}u(\boldsymbol{x}) - u_p^{(1)}(\boldsymbol{x}) \rvert ^2 \mu(\d \boldsymbol{x}) + \int_{\Omega} 2\lvert \frac{1}{2}u(\boldsymbol{x}) - u_p^{(2)}(\boldsymbol{x}) \rvert ^2 \mu(\d \boldsymbol{x}) \\
        & \le \frac{2C^2}{p} + \frac{2C^2}{p} = \frac{4C^2}{p}.\nonumber
        \end{aligned}
    \end{equation}
\end{proof}

\section{Numerical Experiments}
\label{sec:experiments}
In this section, we present numerical experiments to demonstrate the applicability and accuracy of the proposed FENs. The experiments include two key tasks: function approximation and solving PDEs. All experiments are conducted on a high-performance computing server running Debian 12. The server is equipped with an Intel Xeon Platinum 8358 CPU operating at 2.60 GHz, and an NVIDIA A100 GPU with 80 GB of memory. These robust hardware specifications provide the computational capacity required for the intensive operations involved in training and evaluating FENs.

Moreover, to quantitatively evaluate the approximation capabilities of neural networks in the numerical experiments, the maximum absolute error ($L_{\infty}$ error) and the relative $L_2$ error are defined as follows:
\begin{eqnarray}\label{eq:definitionerror}
    e_{L_{\infty}}  = \max_{1\leq i \leq N} \vert u_M(\boldsymbol{x}_i)-u_{exact}(\boldsymbol{x}_i) \vert,  \\
    e_{L_2} = \sqrt{\frac{\sum_{i=1}^{N}(u_M(\boldsymbol{x}_i)-u_{exact}(\boldsymbol{x}_i))^2}{\sum_{i=1}^{N}(u_{exact}(\boldsymbol{x}_i))^2}},
\end{eqnarray}
where $u_M$ and $u_{\text{exact}}$ represent the approximate and exact solutions, respectively, and $\boldsymbol{x}_i$ ($1 \le i \le N$) are the collocation points for error evaluation.

\subsection{Two-dimensional function}
To evaluate the capability of the FENs in function approximation, we implement them on a two-dimensional function defined in Equation \eqref{eq:2D_function}
\begin{equation}
    \label{eq:2D_function}
    u(x, y) = \sin(\pi x) \sin(4\pi y).
\end{equation}

In this experiment, we utilize a uniformly distributed training grid with dimensions $N_x \times N_y = 101 \times 101$, 
where $N_x$ and $N_y$ denote the number of points along the $x$-axis and $y$-axis, respectively. The number of basis functions $M$ is chosen to be $400$, $900$, $1600$, and $2500$, respectively. 
For comparative analysis, we implement ELMs with three distinct activation functions $\text{sigmoid}$, $\tanh$, and $\text{swish}$ to address the same problem.

Table \ref{tab:params_2d} summarizes the parameters used for the approximation. The optimal scaling factor for each activation function is determined by searching within a specified range with a small step size. In Table \ref{tab:error_2d}, we present the approximation errors for FENs and ELMs (using $\text{sigmoid}$, $\tanh$, and $\text{swish}$ activations) 
across different basis function configurations ($M = 400, 900, 1600, 2500$). 
As illustrated in Figure \ref{fig:error_2d}, the $L_{\infty}$ error analysis reveals that FEN solutions exhibit significantly superior accuracy compared to ELM implementations. 
These results demonstrate that the Fourier feature incorporation in FENs provides enhanced approximation capability for the target two-dimensional function, as evidenced by the systematically lower $L_{\infty}$ error metrics.

Quantitatively, when the number of basis functions is sufficient, the smallest $L_{\infty}$ and $L_2$ errors achieved by FENs are $6.4756 \times 10^{-15}$ and $1.4677 \times 10^{-15}$, respectively. In contrast, the smallest $L_{\infty}$ and $L_2$ errors for ELMs are $1.7599 \times 10^{-10}$ and $1.8013 \times 10^{-11}$, respectively. These results highlight the significantly higher representational power of FENs compared to ELMs in this two-dimensional approximation task.

\begin{table}[htp]
    \begin{center}
        \caption{Function approximation: Parameters used in approximating the two-dimensional function defined by Equation \eqref{eq:2D_function}.}
        \setlength\tabcolsep{2pt}
        \small{
        \begin{tabular}{ccccccc}
            \hline\noalign{\smallskip}
            \multirow{2}{*}{Activations} & \multirow{2}{*}{$(\rho_{min}, \rho_{max}]$} & \multirow{2}{*}{$\rho_s$} & \multicolumn{4}{c}{$\rho_{opt}$}  \\
            & &          & $M=400$ & $M=900$ & $M=1600$ & $M=2500$ \\
            \hline
            $\text{sigmoid}$      & $(0, 10]$ & 0.01 & 1.52 & 2.19 & 2.94 & 3.81  \\
            $\tanh$         & $(0, 10]$ & 0.01 & 0.76 & 1.47 & 1.55 & 1.92  \\
            $\text{swish}$        & $(0, 10]$ & 0.01 & 1.49 & 2.61 & 3.06 & 3.76  \\
            $\cos$          & $(0, 50]$ & 0.1  & 7.4  & 9.5  & 9.7  & 22.8  \\
            $\sin$          & $(0, 50]$ & 0.1  & 7.7  & 9.1  & 10.3 & 22.7  \\
            $\cos$ $\&$ $\sin$ & $(0, 50]$ & 0.1  & 9.0  & 9.6  & 9.6  & 12.6  \\
            \hline
        \end{tabular}
        }
        \label{tab:params_2d}
    \end{center}
\end{table}

\begin{figure}[htbp]
    \begin{minipage}{0.9\linewidth}
        \centering
        \includegraphics[width=0.8\textwidth]{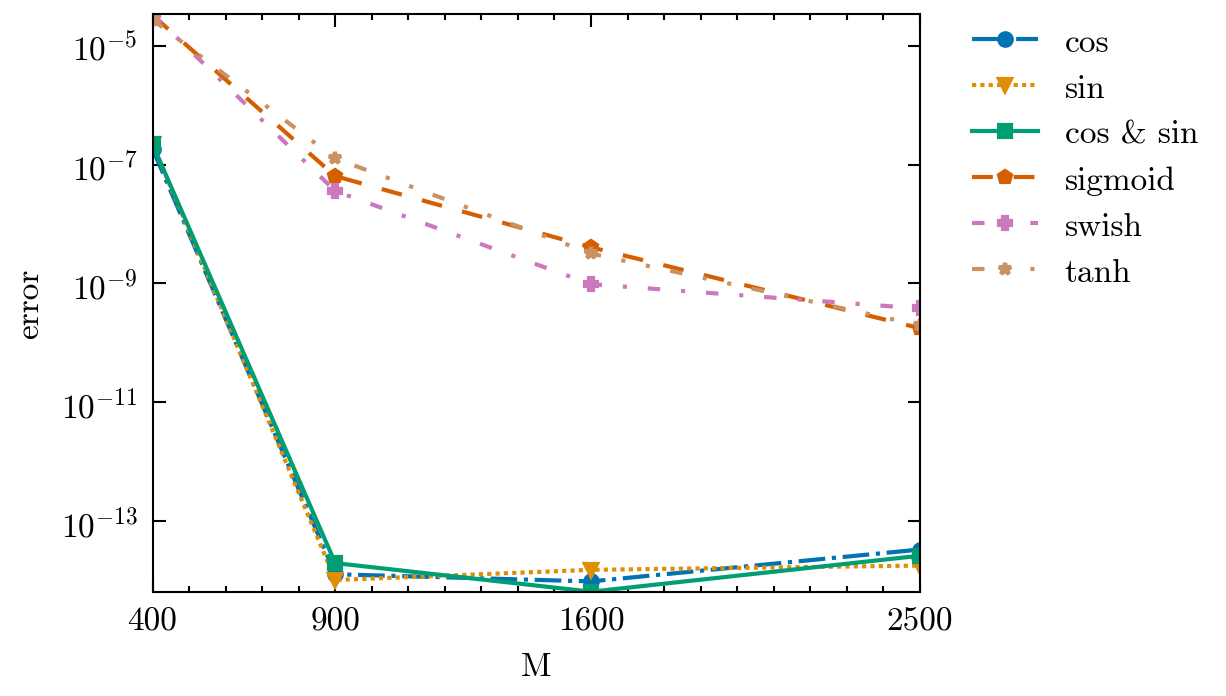}
    \end{minipage}
    \caption{Function approximation: $L_{\infty}$ errors of neural networks in approximating the two-dimensional function defined by Equation \eqref{eq:2D_function}.}
    \label{fig:error_2d}
\end{figure}

\begin{table}[htp]
    \begin{center}
        \caption{Function approximation: Performance comparison of FENs and ELMs with $\text{sigmoid}$, $\tanh$, and $\text{swish}$ activations in approximating the two-dimensional function defined by Equation \eqref{eq:2D_function}. The $L_{\infty}$ and $L_2$ errors for each model configuration are presented.}
        \setlength\tabcolsep{2pt}
        \small{
        \begin{tabular}{ccccccccc}
            \hline\noalign{\smallskip}
            \multirow{2}{*}{Activations} & \multicolumn{2}{c}{M=400} & \multicolumn{2}{c}{M=900} & \multicolumn{2}{c}{M=1600} & \multicolumn{2}{c}{M=2500}  \\
            & $e_{L_{\infty}}$ & $e_{L_{2}}$ & $e_{L_{\infty}}$ & $e_{L_{2}}$ & $e_{L_{\infty}}$ & $e_{L_{2}}$ & $e_{L_{\infty}}$ & $e_{L_{2}}$     \\
            \hline
            \text{sigmoid}      & 3.2425E-05 & 9.2209E-06 &6.3796E-08 & 6.2611E-09 &4.0454E-09 & 2.7714E-10 &1.7599E-10 & 1.8013E-11  \\
            $\tanh$         & 2.7418E-05 & 6.1370E-06 &1.2666E-07 & 1.0780E-08 &3.1869E-09 & 2.4560E-10 &1.9782E-10 & 1.8775E-11  \\
            $\text{swish}$        & 3.3855E-05 & 1.0940E-05 &3.5856E-08 & 4.5318E-09 &9.6770E-10 & 1.6892E-10 &3.7835E-10 & 2.0718E-11  \\
            $\cos$          & 1.8362E-07 & 7.4465E-08 &1.2727E-14 & 2.3325E-15 &9.6481E-15 & 1.6508E-15 &3.3154E-14 & 3.5549E-15  \\
            $\sin$          & 2.2316E-07 & 9.5567E-08 &1.0184E-14 & 3.9475E-15 &1.5099E-14 & 2.0114E-15 &1.7819E-14 & 3.4593E-15  \\
            $\cos$ $\&$ $\sin$ & 2.1903E-07 & 7.8576E-08 &1.9601E-14 & 4.6683E-15 &6.4756E-15 & 1.4677E-15 &2.6069E-14 & 2.3770E-15  \\
            \hline
        \end{tabular}
        }
        \label{tab:error_2d}
    \end{center}
\end{table}

\subsection{Helmholtz equation}
 The Helmholtz equation in two dimensions is given by Equation~\eqref{eq:Helmholtz_equation} 
\begin{equation}
    \label{eq:Helmholtz_equation}
    \begin{array}{r@{}l}
        \left\{
        \begin{aligned}
            \Delta u + k^2 u & = q, &  & \mbox{in} \enspace \Omega,          \\
            u             & = h,        &  & \mbox{on} \enspace \partial \Omega,
        \end{aligned}
        \right.
    \end{array}
\end{equation}
where $\Omega = (0, 1)^2$. 
The exact solution is defined as 
\begin{equation}
    \label{eq:Helmholtz_solution}
    u(x, y) = \sin(a_1 \pi x) \sin(a_2 \pi y),
\end{equation}
with the corresponding source term given by 
\begin{equation}
    \label{eq:Helmholtz_source_term}
    q(x, y) = \left(k^2-(a_1 \pi)^2 - (a_2 \pi)^2\right) \sin(a_1 \pi x) \sin(a_2 \pi y),
\end{equation}
where the parameters are set as $a_1 = 1$, $a_2 = 4$, and $k = 1$.

For training both FENs and ELMs, we employ a uniform grid of $N_x \times N_y = 101 \times 101$ collocation points. This training dataset comprises both interior points within $\Omega$ and boundary points on $\partial \Omega$. Consequently, the dataset $S$ is divided into two distinct subsets: $S_r$ (containing interior points), and $S_b$ (containing boundary points).

Table~\ref{tab:params_Helmholtz2D} summarizes the optimal scaling factors and search parameters used in solving the Helmholtz equation. Table~\ref{tab:error_Helmholtz2D} presents the approximation errors of FENs and ELMs with varying numbers of basis functions, all utilizing the optimized scaling configuration. Figure~\ref{fig:error_Helmholtz2D} illustrates the $L_{\infty}$ error trajectories for all models, demonstrating conclusively that FENs achieve superior approximation accuracy compared to ELMs across all parametric configurations.

The data presented in Table~\ref{tab:error_Helmholtz2D} reveal that with a sufficient number of basis functions, FENs attain minimum 
$L_{\infty}$ and $L_2$ errors of $5.3300 \times 10^{-14}$ and $2.7200 \times 10^{-14}$, respectively. In contrast, the best performance from ELMs yields $L_{\infty}$ and $L_2$ errors of $4.9477 \times 10^{-10}$ and $2.7595 \times 10^{-10}$, respectively. This result clearly demonstrates that FENs possess superior representational capacity and can achieve significantly higher accuracy than ELMs for the Helmholtz boundary value problem.

\begin{table}[htp]
    \begin{center}
        \caption{Helmholtz equation: Parameters when solving the Helmholtz equation \eqref{eq:Helmholtz_equation} with solution in Equation \eqref{eq:Helmholtz_solution}.}
        \setlength\tabcolsep{2pt}
        \small{
        \begin{tabular}{ccccccc}
            \hline\noalign{\smallskip}
            \multirow{2}{*}{Activations} & \multirow{2}{*}{$(\rho_{min}, \rho_{max}]$} & \multirow{2}{*}{$\rho_s$} & \multicolumn{4}{c}{$\rho_{opt}$}  \\
            & &          & $M=400$ & $M=900$ & $M=1600$ & $M=2500$ \\
            \hline
            $\text{sigmoid}$      & $(0, 10]$  & 0.01 & 1.38 & 1.66 & 2.48 & 2.98  \\
            $\tanh$         & $(0, 10]$  & 0.01 & 0.6  & 0.94 & 1.1  & 1.5  \\
            $\text{swish}$        & $(0, 50]$  & 0.01 & 1.4  & 2.0  & 2.2  & 2.4  \\
            $\cos$          & $(0, 100]$ & 0.1  & 5.6  & 9.3  & 9.0  & 13.1  \\
            $\sin$          & $(0, 100]$ & 0.1  & 5.8  & 9.9  & 9.9  & 12.0  \\
            $\cos$ $\&$ $\sin$ & $(0, 100]$ & 0.1  & 5.5  & 10.5 & 10.1 & 8.8  \\
            \hline
        \end{tabular}
        }
        \label{tab:params_Helmholtz2D}
    \end{center}
\end{table}

\begin{figure}[htbp]
    \begin{minipage}{0.9\linewidth}
        \centering
        \includegraphics[width=0.8\textwidth]{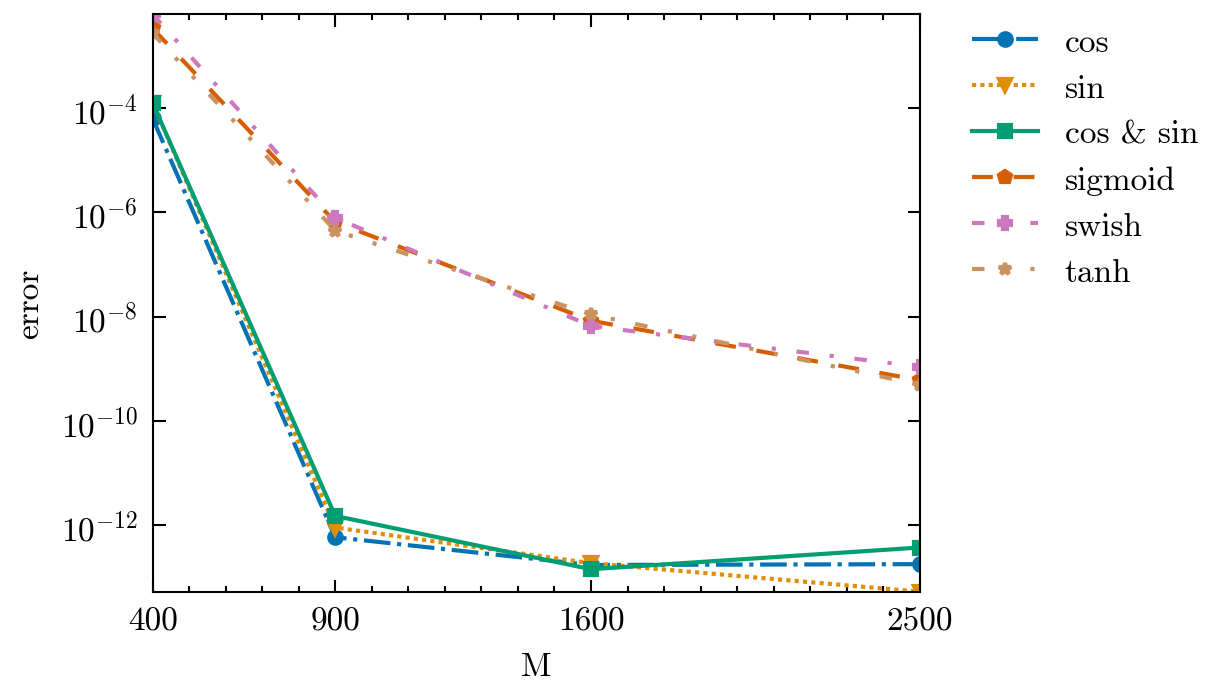}
    \end{minipage}    
    \caption{Helmholtz equation: $L_{\infty}$ errors of neural networks when solving Helmholtz equation \eqref{eq:Helmholtz_equation} with solution in Equation \eqref{eq:Helmholtz_solution}.}
    \label{fig:error_Helmholtz2D}
\end{figure}

\begin{table}[htp]
    \begin{center}
        \caption{Helmholtz equation: Performance comparison of FENs and ELMs activated by $\text{sigmoid}$, $\tanh$ and $\text{swish}$ when solving the Helmholtz equation \eqref{eq:Helmholtz_equation} with solution in Equation \eqref{eq:Helmholtz_solution}. The $L_{\infty}$ errors and $L_{2}$ errors for each model configuration are presented.}
        \setlength\tabcolsep{2pt}
        \small{
        \begin{tabular}{ccccccccc}
            \hline\noalign{\smallskip}
            \multirow{2}{*}{Activations} & \multicolumn{2}{c}{M=400} & \multicolumn{2}{c}{M=900} & \multicolumn{2}{c}{M=1600} & \multicolumn{2}{c}{M=2500}  \\
            & $e_{L_{\infty}}$ & $e_{L_{2}}$ & $e_{L_{\infty}}$ & $e_{L_{2}}$ & $e_{L_{\infty}}$ & $e_{L_{2}}$ & $e_{L_{\infty}}$ & $e_{L_{2}}$     \\
            \hline
            $\text{sigmoid}$      & 3.3875E-03 & 8.3152E-04 &6.5565E-07 & 5.7295E-07 &8.3819E-09 & 2.8206E-09 &6.0754E-10 & 3.9546E-10  \\
            $\tanh$         & 2.7466E-03 & 6.7760E-04 &4.4703E-07 & 2.0869E-07 &1.1059E-08 & 6.0855E-09 &4.9477E-10 & 2.7595E-10  \\
            $\text{swish}$        & 6.3477E-03 & 1.5769E-03 &7.8976E-07 & 2.2362E-07 &6.6357E-09 & 3.0735E-09 &1.0987E-09 & 3.9585E-10  \\
            $\cos$          & 6.3797E-05 & 1.7419E-05 &5.8978E-13 & 1.5736E-13 &1.7333E-13 & 7.0401E-14 &1.7977E-13 & 7.1273E-14  \\
            $\sin$          & 1.2699E-04 & 1.8981E-05 &9.0587E-13 & 2.7940E-13 &1.8736E-13 & 7.4318E-14 &5.3300E-14 & 2.7200E-14  \\
            $\cos$ $\&$ $\sin$ & 1.2552E-04 & 2.2074E-05 &1.5099E-12 & 2.4098E-13 &1.4385E-13 & 6.6964E-14 &3.7487E-13 & 1.3203E-13  \\
            \hline
        \end{tabular}
        }
        \label{tab:error_Helmholtz2D}
    \end{center}
\end{table}

\chadded{
For comparison, we also employ the finite element method (FEM) to solve the Helmholtz equation using a triangulation‌ with different mesh sizes $h$. 
The corresponding numerical results are presented in Table \ref{tab:error_Helmholtz2D_FEM}. It can be observed that even with a large number of degree of freedom (DoF), the accuracy achieved by FEM is still much lower than that of FENs, which is shown in Table \ref{tab:error_Helmholtz2D}.
}

\chadded{
\begin{remark}
    In this work, FENs employ a single-hidden-layer neural network with trigonometric activation functions as global basis functions, and solve PDEs using the collocation method. In contrast, the FEM uses local basis functions to approximate the solution piecewise, and is based on the variational principle to obtain a weak-form solution of the PDEs, whereas FENs aim to solve the strong form solution. While spectral methods also use global basis functions to approximate the strong form solution, they typically require the basis functions to be orthogonal. A key distinction between spectral methods and FENs lies in how they construct basis functions for multi-dimensional problems. For 2D or 3D PDEs, FENs can generate basis functions simply by randomly sampling the hidden-layer weights as in the 1D case, whereas spectral methods may generate basis functions using tensor products of 1D orthogonal basis functions.
\end{remark}
}

\begin{table}[htp]
    \begin{center}
        \caption{Helmholtz equation: Results of FEM when solving the Helmholtz equation \eqref{eq:Helmholtz_equation} with solution in Equation \eqref{eq:Helmholtz_solution}. The $L_{\infty}$ errors and $L_{2}$ errors for each model configuration are presented.}
        \setlength\tabcolsep{2pt}
        \small{
        \begin{tabular}{ccccc}
            \hline\noalign{\smallskip}
            Method & $h$ & DoF & $e_{L_{\infty}}$ & $e_{L_{2}}$  \\
            \hline
            \multirow{5}{*}{FEM} & 0.1 & 1969 & 1.5408E-02 & 7.4130E-03 \\
              & 0.05 & 7577 & 1.5028E-03 & 8.9865E-04 \\
              & 0.02 & 46905 & 1.4795E-04 & 4.1033E-05 \\
              & 0.01 & 185913 & 1.8508E-05 & 5.0898E-06 \\
              & 0.005 & 741209 & 1.3809E-06 & 9.4240E-07 \\
            \hline
        \end{tabular}
        }
        \label{tab:error_Helmholtz2D_FEM}
    \end{center}
\end{table}


\chadded{
To further assess the approximation capabilities of FENs and to compare with cos/sin-based ELMs, we consider the following solution of Helmholtz equation \eqref{eq:Helmholtz_equation},
\begin{equation}
    \label{eq:Helmholtz_solution_2}
    u(x, y) = \tanh(xy),
\end{equation}
with the associated source term defined by
\begin{equation}
    \label{eq:Helmholtz_source_term_2}
    q(x, y)=(k^2-2x^2-2y^2) \tanh(xy) + (2x^2+2y^2) \tanh^3(xy).
\end{equation}
}
\chadded{
Unlike the  solution in Equation~\eqref{eq:Helmholtz_solution}, this exact function is closely related to the $\tanh$ activation function. Despite this, as shown in Table~\ref{tab:error_Helmholtz2D_2}, FENs still outperform ELMs activated by non-trigonometric function, particularly when the number of basis functions is small. The minimum $L_{\infty}$ and $L_2$ errors achieved by FENs are $3.2613 \times 10^{-15}$ and $2.4689 \times 10^{-15}$, respectively, compared to $4.1078 \times 10^{-15}$ and $3.8454 \times 10^{-15}$ achieved by ELMs with trigonometric activation function. They both achieve precise results, which demonstrate that activation function is very important during the computation.
}

\begin{table}[htp]
    \begin{center}
        \caption{Helmholtz equation: Parameters when solving the Helmholtz equation \eqref{eq:Helmholtz_equation} with solution in Equation \eqref{eq:Helmholtz_solution_2}.}
        \setlength\tabcolsep{2pt}
        \small{
        \begin{tabular}{ccccccc}
            \hline\noalign{\smallskip}
            \multirow{2}{*}{Activations} & \multirow{2}{*}{$(\rho_{min}, \rho_{max}]$} & \multirow{2}{*}{$\rho_s$} & \multicolumn{4}{c}{$\rho_{opt}$}  \\
            & &          & $M=400$ & $M=900$ & $M=1600$ & $M=2500$ \\
            \hline
            $\text{sigmoid}$      & $(0, 10]$  & 0.01 & 1.51 & 2.27 & 2.47 & 3.31  \\
            $\tanh$         & $(0, 10]$  & 0.01 & 0.73 & 1.06 & 1.53 & 1.59  \\
            $\text{swish}$        & $(0, 50]$  & 0.01 & 1.6  & 2.1  & 2.8  & 4.0  \\
            $\cos$          & $(0, 100]$ & 0.1  & 7.6  & 9.4  & 10.2 & 12.1  \\
            $\sin$          & $(0, 100]$ & 0.1  & 8.3  & 9.2  & 10.3 & 15.9  \\
            $\cos$ $\&$ $\sin$ & $(0, 100]$ & 0.1  & 7.5  & 9.7  & 11.2 & 23.6  \\
            \hline
        \end{tabular}
        }
        \label{tab:params_Helmholtz2D_2}
    \end{center}
\end{table}


\begin{table}[htp]
    \begin{center}
        \caption{Helmholtz equation: Performance comparison of FENs and ELMs activated by different functions, when solving the Helmholtz equation \eqref{eq:Helmholtz_equation} with solution \eqref{eq:Helmholtz_solution_2}. The $L_{\infty}$ errors and $L_{2}$ errors for each model configuration are presented.}
        \setlength\tabcolsep{2pt}
        \small{
        \begin{tabular}{ccccccccc}
            \hline\noalign{\smallskip}
            \multirow{2}{*}{Methods} & \multicolumn{2}{c}{M=400} & \multicolumn{2}{c}{M=900} & \multicolumn{2}{c}{M=1600} & \multicolumn{2}{c}{M=2500}  \\
            & $e_{L_{\infty}}$ & $e_{L_{2}}$ & $e_{L_{\infty}}$ & $e_{L_{2}}$ & $e_{L_{\infty}}$ & $e_{L_{2}}$ & $e_{L_{\infty}}$ & $e_{L_{2}}$     \\
            \hline
           ELMs ($\text{sigmoid})$      & 2.9184E-08 & 1.5916E-08 & 2.5498E-11 & 1.5287E-11 & 2.4958E-13 & 1.5588E-13 & 2.0040E-14 & 2.1436E-14  \\
           ELMs ($\text{tanh})$         & 2.3647E-08 & 1.2001E-08 & 1.6025E-11 & 1.0172E-11 & 1.8674E-13 & 1.5324E-13 & 1.1768E-14 & 9.1806E-15  \\
            ELMs ($\text{swish})$        & 2.5343E-08 & 1.0115E-08 & 1.7768E-11 & 1.8566E-11 & 3.4794E-13 & 3.9641E-13 & 3.8497E-14 & 5.8373E-14  \\
            ELMs ($\sin$) & 1.0436E-09 & 4.3592E-10 & 1.6043E-14 & 8.7617E-15 & 4.1078E-15 & 3.8454E-15 & 4.2188E-15 & 3.9039E-15  \\
              FENs ($\cos$ $\&$ $\sin$) & 4.0931E-10 & 1.6929E-10 & 1.2101E-14 & 9.6805E-15 & 3.2613E-15 & 2.4689E-15 & 3.7192E-15 & 2.2473E-15  \\
            \hline
        \end{tabular}
        }
        \label{tab:error_Helmholtz2D_2}
    \end{center}
\end{table}


\subsection{Diffusion equation}
In this subsection, we investigate the behavior of a two-dimensional diffusion equation that involves both spatial and temporal variables.
The initial-boundary value problem is governed by
\begin{equation}
    \label{eq:Diffusion_equation}
    \begin{array}{r@{}l}
        \left\{
        \begin{aligned}
            \frac{\partial u}{\partial t} - \nu \frac{\partial^2 u}{\partial x^2} & = f(x, t), &  & (x, t) \in \ (a_1, b_1) \times (0, t_f],          \\
            u(a_1, t)             & = g_1(t),        &  & t \in(0, t_f], \\
            u(b_1, t)             & = g_2(t),        &  & t \in(0, t_f], \\
            u(x, 0)             & = h(x),        &  & x \in [a_1,b_1], \\
        \end{aligned}
        \right.
    \end{array}
\end{equation}
where $f(x ,t)$ denotes the source term, $\nu > 0$ denotes the constant diffusion coefficient, $g_1(t)$ and $g_2(t)$ prescribe the time-dependent Dirichlet boundary conditions, and $h(x)$ defines the initial condition. 
The parameters are chosen as $a_1=0$, $b_1=5$ and $\nu=0.01$. The final time $t_f$ considered in simulations is set to be $1$. 

We choose the suitable functions $f(x,t)$, $g_1(t)$, $g_2(t)$, and the initial condition $h(x)$ so that the exact solution is given by
\begin{equation}
    \label{eq:Diffusion_equation_solution}
    u(x, t)=\left[2 \cos\left(\pi x+\frac{\pi}{5}\right) + \frac{3}{2}\cos\left(2\pi x - \frac{3\pi}{5}\right)\right]\left[2 \cos\left(\pi t+\frac{\pi}{5}\right) + \frac{3}{2} \cos\left(2\pi t - \frac{3\pi}{5}\right)\right].
\end{equation}
By treating time as a spatial dimension, we convert the 1D time-dependent diffusion equation into a 2D problem. Both FENs and ELMs are trained on a $N_x \times N_y = 101 \times 101$ collocation points. 
Table~\ref{tab:params_Diffusion} lists the relevant parameters and the optimal scaling factor used in this setup. Table~\ref{tab:error_Diffusion} compares the approximate errors of the solutions obtained by FENs and ELMs for different numbers of basis functions. Figure~\ref{fig:error_Diffusion} illustrates the $L_{\infty}$ error curves, showing significantly lower errors for FENs than ELMs.

Furthermore, by examining Table~\ref{tab:error_Diffusion}, we observe that the minimal  $L_{\infty}$ and $L_2$ errors achieved by FENs are $9.2371 \times 10^{-14}$ and $5.9186 \times 10^{-15}$, respectively. In contrast, the minimal  $L_{\infty}$ and $L_2$ errors obtained by ELMs are $5.7246 \times 10^{-9}$ and $3.8524 \times 10^{-10}$, respectively. These results demonstrate that FENs are more suitable for solving this problem compared to ELMs.

\begin{table}[htp]
    \begin{center}
        \caption{Diffusion equation: Parameters when solving the diffusion equation \eqref{eq:Diffusion_equation} with $t_f=1$.}
        \setlength\tabcolsep{2pt}
        \small{
        \begin{tabular}{ccccccc}
            \hline\noalign{\smallskip}
            \multirow{2}{*}{Activations} & \multirow{2}{*}{$(\rho_{min}, \rho_{max}]$} & \multirow{2}{*}{$\rho_s$} & \multicolumn{4}{c}{$\rho_{opt}$}  \\
            & &          & $M=400$ & $M=900$ & $M=1600$ & $M=2500$ \\
            \hline
            $\text{sigmoid}$      & $(0, 10]$  & 0.01 & 1.55 & 1.86 & 2.61 & 3.35  \\
            $\tanh$         & $(0, 10]$  & 0.01 & 0.9  & 1.01 & 1.38 & 2.03  \\
            $\text{swish}$        & $(0, 50]$  & 0.01 & 1.55 & 2.5  & 2.53 & 3.96  \\
            $\cos$          & $(0, 100]$ & 0.1  & 5.5  & 5.0  & 5.4  & 6.4  \\
            $\sin$          & $(0, 100]$ & 0.1  & 4.0  & 5.5  & 5.3  & 6.6  \\
            $\cos$ $\&$ $\sin$ & $(0, 100]$ &  0.1 & 5.6  & 5.2  & 5.4  & 5.3  \\
            \hline
        \end{tabular}
        }
        \label{tab:params_Diffusion}
    \end{center}
\end{table}

\begin{figure}[htbp]
    \begin{minipage}{0.9\linewidth}
        \centering
        \includegraphics[width=0.8\textwidth]{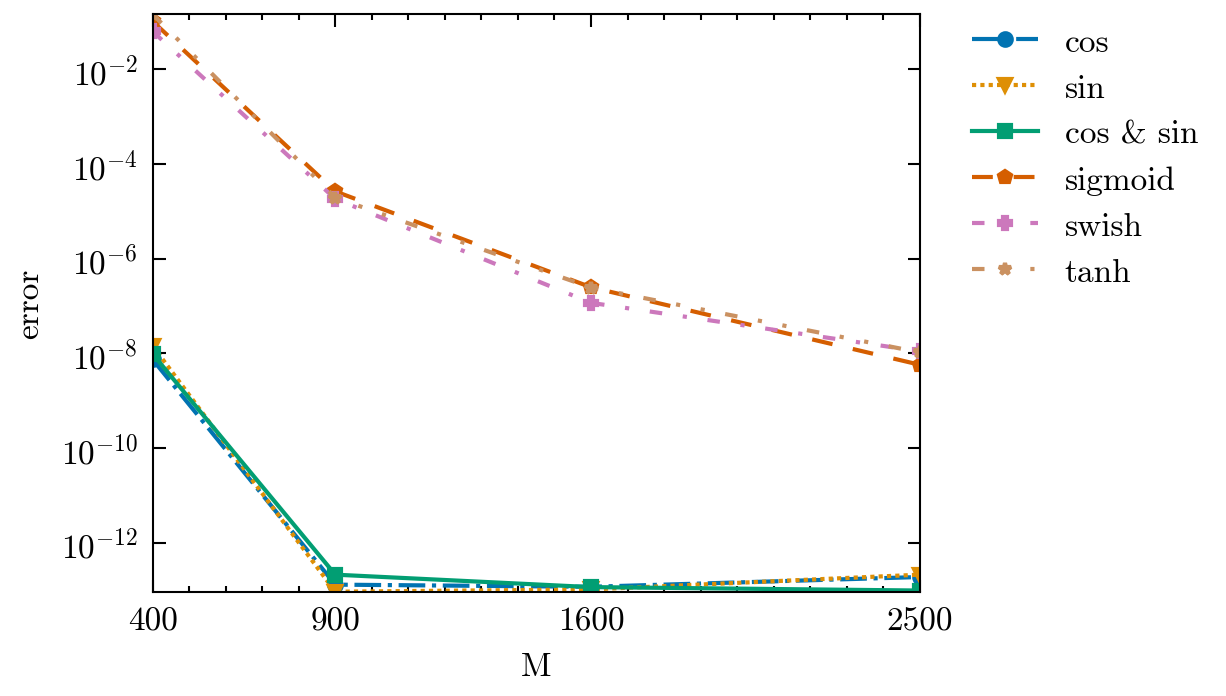}
    \end{minipage}
    \caption{Diffusion equation: the $L_{\infty}$ errors of neural networks when solving diffusion equation \eqref{eq:Diffusion_equation} with $t_f=1$.}
    \label{fig:error_Diffusion}
\end{figure}

\begin{table}[htp]
    \begin{center}
        \caption{Diffusion equation: Performance comparison of FENs and ELMs activated by $\text{sigmoid}$, $\tanh$ and $\text{swish}$ when solving the diffusion equation \eqref{eq:Diffusion_equation} with $t_f=1$.  The $L_{\infty}$ errors and $L_{2}$ errors for each model configuration are presented.}
        \setlength\tabcolsep{2pt}
        \small{
        \begin{tabular}{ccccccccc}
            \hline\noalign{\smallskip}
            \multirow{2}{*}{Activations} & \multicolumn{2}{c}{M=400} & \multicolumn{2}{c}{M=900} & \multicolumn{2}{c}{M=1600} & \multicolumn{2}{c}{M=2500}  \\
            & $e_{L_{\infty}}$ & $e_{L_{2}}$ & $e_{L_{\infty}}$ & $e_{L_{2}}$ & $e_{L_{\infty}}$ & $e_{L_{2}}$ & $e_{L_{\infty}}$ & $e_{L_{2}}$     \\
            \hline
            $\text{sigmoid}$      & 1.0276E-01 & 7.3530E-03 & 2.6306E-05 & 2.4873E-06 & 2.4815E-07 & 1.8386E-08 & 5.7246E-09 & 3.8524E-10  \\
            $\tanh$         & 1.4518E-01 & 7.2468E-03 & 1.9460E-05 & 1.1581E-06 & 2.4414E-07 & 1.7353E-08 & 1.0454E-08 & 4.3533E-10  \\
            $\text{swish}$        & 6.3047E-02 & 4.5548E-03 & 1.8239E-05 & 1.5684E-06 & 1.1645E-07 & 6.2690E-09 & 1.1023E-08 & 9.1794E-10  \\
            $\cos$          & 7.4218E-09 & 5.5655E-10 & 1.3012E-13 & 5.9433E-15 & 1.1724E-13 & 9.1437E-15 & 1.8852E-13 & 9.1456E-15  \\
            $\sin$          & 1.4405E-08 & 8.5286E-10 & 9.2371E-14 & 5.9186E-15 & 1.0303E-13 & 8.3162E-15 & 2.1139E-13 & 6.9327E-15  \\
            $\cos$ $\&$ $\sin$ & 9.6760E-09 & 1.3414E-10 & 2.1139E-13 & 1.1447E-14 & 1.1546E-13 & 1.3558E-14 & 9.7700E-14 & 8.3418E-15  \\
            \hline
        \end{tabular}
        }
        \label{tab:error_Diffusion}
    \end{center}
\end{table}
\subsection{Heat equation}
We consider the heat equation within the spatial-temporal domain $\Omega \times (0, t_f] = (0, 1)^2 \times (0, 1]$, governed by the following system of PDEs
\begin{equation}
    \label{eq:Heat_equation}
    \begin{array}{r@{}l}
        \left\{
        \begin{aligned}
            u_t(x,y,t) -\Delta u(x,y,t) & = f(x,y,t), &  & (x,y,t) \in \Omega  \times (0, 1],          \\
            u(x,y,t)             & = g(x,y,t),        &  & (x,y,t) \in \partial \Omega \times (0, 1], \\
            u(x, y, 0)             & = h(x, y),        &  & (x, y) \in \Omega. \\
        \end{aligned}
        \right.
    \end{array}
\end{equation}
We choose the suitable functions $f(x,y,t)$, $g(x,y,t)$, and the initial condition $h(x,y,t)$ so that the exact solution is given by
$$u(x, y, t)=2e^{-t} \sin\left(\frac{\pi}{2}x\right) \sin\left(\frac{\pi}{2}y\right).$$

For the training phase, a uniform grid of $N_x \times N_y \times N_t = 51 \times 51 \times 51$ collocation points is employed. The parameters and results from the optimal scale search are summarized in Table~\ref{tab:params_Heat}. Table~\ref{tab:error_Heat} presents the $L_{\infty}$ and $L_2$ errors for both FENs and ELMs across varying numbers of basis functions. The $L_{\infty}$ error curves depicted in Figure~\ref{fig:error_Heat} clearly demonstrate that FENs consistently yield lower approximation errors compared to ELMs, highlighting their superior performance in solving the heat equation.

In terms of specific approximation error values, when the number of basis functions is sufficiently large, FENs can achieve the smallest $L_{\infty}$ and $L_2$ errors of $2.6090 \times 10^{-14}$ and $1.7562 \times 10^{-14}$, respectively. In contrast, ELMs reach their smallest $L_{\infty}$ and $L_2$ errors of $2.4443 \times 10^{-12}$ and $1.4761 \times 10^{-12}$, respectively. This significant difference in error magnitudes clearly demonstrates that FENs can provide solutions with considerably higher precision than ELMs, making them a more suitable method for solving this problem.

\begin{table}[htp]
    \begin{center}
        \caption{Heat equation: Parameters when solving the heat equation \eqref{eq:Heat_equation}.}
        \setlength\tabcolsep{2pt}
        \small{
        \begin{tabular}{ccccccc}
            \hline\noalign{\smallskip}
            \multirow{2}{*}{Activations} & \multirow{2}{*}{$(\rho_{min}, \rho_{max}]$} & \multirow{2}{*}{$\rho_s$} & \multicolumn{4}{c}{$\rho_{opt}$}  \\
            & &          & $M=400$ & $M=900$ & $M=1600$ & $M=2500$ \\
            \hline
            $\text{sigmoid}$      & $(0, 5]$  & 0.01 & 0.15 & 0.3  & 0.46 & 0.56  \\
            $\tanh$         & $(0, 5]$  & 0.01 & 0.07 & 0.14 & 0.19 & 0.22  \\
            $\text{swish}$        & $(0, 5]$  & 0.01 & 0.12 & 0.21 & 0.31 & 0.48  \\
            $\cos$          & $(0, 50]$  & 0.1  & 1.1  & 1.4  & 2.2  & 1.7  \\
            $\sin$          & $(0, 50]$  & 0.1  & 1.2  & 1.3  & 1.6  & 2.4  \\
            $\cos$ $\&$ $\sin$ & $(0, 50]$  & 0.1  & 1.1  & 1.3  & 1.6  & 2.0  \\
            \hline
        \end{tabular}
        }
        \label{tab:params_Heat}
    \end{center}
\end{table}

\begin{figure}[htbp]
    \begin{minipage}{0.9\linewidth}
        \centering
        \includegraphics[width=0.8\textwidth]{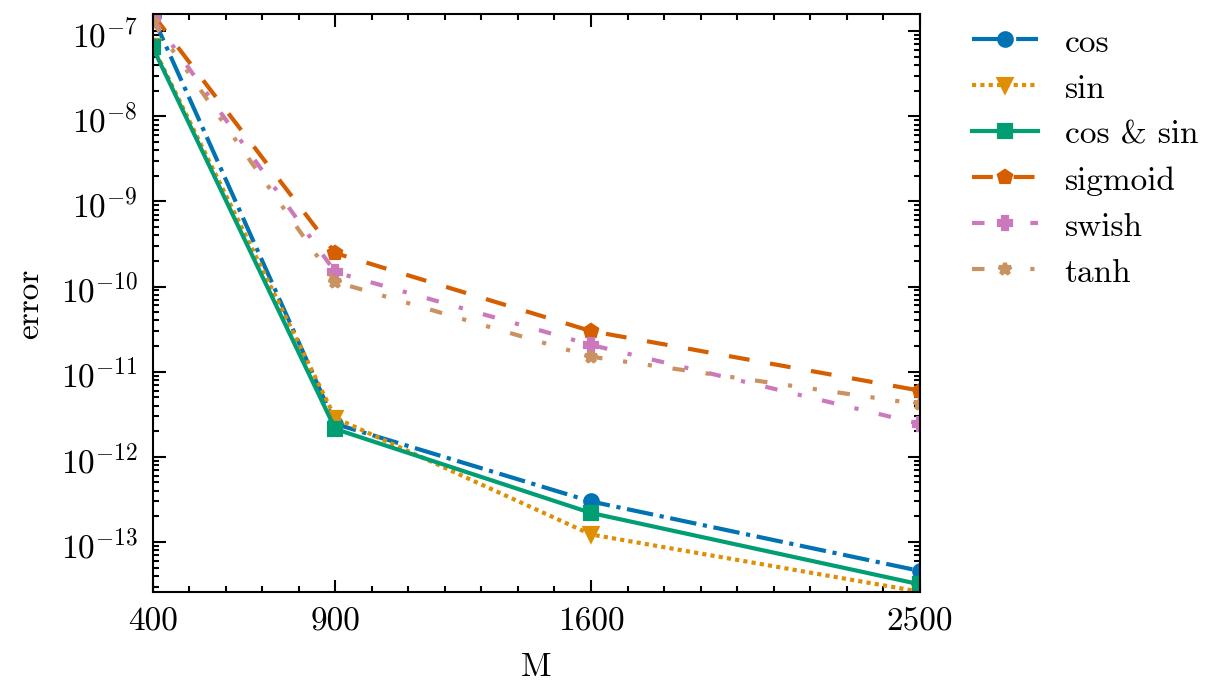}
    \end{minipage}
    \caption{Heat equation: $L_{\infty}$ errors of neural networks when solving the heat equation \eqref{eq:Heat_equation}.}
    \label{fig:error_Heat}
\end{figure}

\begin{table}[htp]
    \begin{center}
        \caption{Heat equation: Performance comparison of FENs and ELMs activated by $\text{sigmoid}$, $\tanh$ and $\text{swish}$ when solving the heat equation \eqref{eq:Heat_equation}. The $L_{\infty}$ errors and $L_{2}$ errors for each model configuration are presented.}
        \setlength\tabcolsep{2pt}
        \small{
        \begin{tabular}{ccccccccc}
            \hline\noalign{\smallskip}
            \multirow{2}{*}{Activations} & \multicolumn{2}{c}{M=400} & \multicolumn{2}{c}{M=900} & \multicolumn{2}{c}{M=1600} & \multicolumn{2}{c}{M=2500}  \\
            & $e_{L_{\infty}}$ & $e_{L_{2}}$ & $e_{L_{\infty}}$ & $e_{L_{2}}$ & $e_{L_{\infty}}$ & $e_{L_{2}}$ & $e_{L_{\infty}}$ & $e_{L_{2}}$     \\
            \hline
            $\text{sigmoid}$      & 1.5832E-07 & 7.3817E-08 & 2.4738E-10 & 2.3643E-10 & 3.0013E-11 & 2.2705E-11 & 5.9936E-12 & 3.5015E-12  \\
            $\tanh$         & 1.2910E-07 & 5.3362E-08 & 1.1386E-10 & 5.7319E-11 & 1.5007E-11 & 7.7614E-12 & 4.1442E-12 & 3.0037E-12  \\
            $\text{swish}$        & 1.4388E-07 & 7.5731E-08 & 1.4778E-10 & 7.2960E-11 & 2.0669E-11 & 8.1287E-12 & 2.4443E-12 & 1.4761E-12  \\
            $\cos$          & 1.5015E-07 & 4.9863E-08 & 2.4385E-12 & 1.2761E-12 & 2.9932E-13 & 1.2883E-13 & 4.5741E-14 & 2.8907E-14  \\
            $\sin$          & 6.6993E-08 & 3.8847E-08 & 2.8706E-12 & 1.1786E-12 & 1.2257E-13 & 9.7663E-14 & 2.6090E-14 & 1.7562E-14  \\
            $\cos$ $\&$ $\sin$ & 6.4772E-08 & 2.3952E-08 & 2.1253E-12 & 9.1364E-13 & 2.1938E-13 & 1.3437E-13 & 3.1822E-14 & 3.0063E-14  \\
            \hline
        \end{tabular}
        }
        \label{tab:error_Heat}
    \end{center}
\end{table}

\subsection{Wave equation}
We consider the following wave equation defined over the spatial domain $\Omega = (0, 1)^2$ and the temporal domain $(0, 1]$, 
\begin{equation}
    \label{eq:Wave_equation}
    \begin{array}{r@{}l}
        \left\{
        \begin{aligned}
            \frac{\partial ^2 u}{\partial t^2} -\Delta u(x,y,t) & = f(x,y,t), &  & (x,y,t) \in \Omega  \times (0, 1],          \\
            u(x,y,t)             & = g(x,y,t),        &  & (x,y,t) \in \partial \Omega \times (0, 1], \\
            u(x, y, 0)             & = h(x, y),        &  & (x, y) \in \Omega, \\
            \frac{\partial u}{\partial t} (x, y, 0) & = w(x, y), & & (x, y) \in \Omega. \\
        \end{aligned}
        \right.
    \end{array}
\end{equation}
The exact solution is given by $u(x, y, t)=\sin\left(\frac{\pi}{2}x\right) \sin\left(\frac{\pi}{2}y\right) \sin\left(\frac{\pi}{2}t\right)$, 
with suitable boundary condition $g(x, y, t)$, initial conditions $h(x, y, t)$ and $w(x, y)$, as well as the source term $f(x, y, t)$.
We employ a uniform grid of $N_x \times N_y \times N_t = 51 \times 51 \times 51$ collocation points.

Following the same procedure as in previous sections, we conduct an optimal scale search across varying numbers of basis functions. The corresponding parameters are listed in Table~\ref{tab:params_Wave}. The resulting errors obtained using FENs and ELMs with their respective optimal scaling factors are presented in Table~\ref{tab:error_Wave}. Figure~\ref{fig:error_Wave} illustrates the $L_{\infty}$ and $L_2$ error curves for both computational approaches, clearly showing that FENs maintain superior error-reduction capabilities throughout all test configurations when benchmarked against ELMs.

Furthermore, as quantitatively demonstrated in Table~\ref{tab:error_Wave}, the minimum $L_{\infty}$ and $L_2$ errors achieved by FENs are $2.4425 \times 10^{-15}$ and $1.7659 \times 10^{-15}$, respectively. In contrast, minimum $L_{\infty}$ and $L_2$ errors achieved by ELMs are $1.4412 \times 10^{-12}$ and $5.1755 \times 10^{-13}$, respectively. This systematic comparison conclusively establishes FENs' superior efficacy in wave equation solutions, with demonstrably higher computational precision compared to ELMs. 

\begin{table}[htp]
    \begin{center}
        \caption{Wave equation: Parameters when solving the wave equation \eqref{eq:Wave_equation}.}
        \setlength\tabcolsep{2pt}
        \small{
        \begin{tabular}{ccccccc}
            \hline\noalign{\smallskip}
            \multirow{2}{*}{Activations} & \multirow{2}{*}{$(\rho_{min}, \rho_{max}]$} & \multirow{2}{*}{$\rho_s$} & \multicolumn{4}{c}{$\rho_{opt}$}  \\
            & &          & $M=400$ & $M=900$ & $M=1600$ & $M=2500$ \\
            \hline
            $\text{sigmoid}$      & $(0, 5]$   & 0.01 & 0.17 & 0.27 & 0.5  & 0.54  \\
            $\tanh$         & $(0, 5]$   & 0.01 & 0.06 & 0.13 & 0.18 & 0.24  \\
            $\text{swish}$        & $(0, 5]$   & 0.01 & 0.13 & 0.24 & 0.35 & 0.49  \\
            $\cos$          & $(0, 10]$  & 0.1  & 1.28 & 1.34 & 1.46 & 1.7  \\
            $\sin$          & $(0, 10]$  & 0.1  & 1.3  & 1.38 & 1.47 & 1.39  \\
            $\cos$ $\&$ $\sin$ & $(0, 10]$  & 0.1  & 1.16 & 1.28 & 1.62 & 1.74  \\
            \hline
        \end{tabular}
        }
        \label{tab:params_Wave}
    \end{center}
\end{table}

\begin{figure}[htbp]
    \begin{minipage}{0.9\linewidth}
        \centering
        \includegraphics[width=0.8\textwidth]{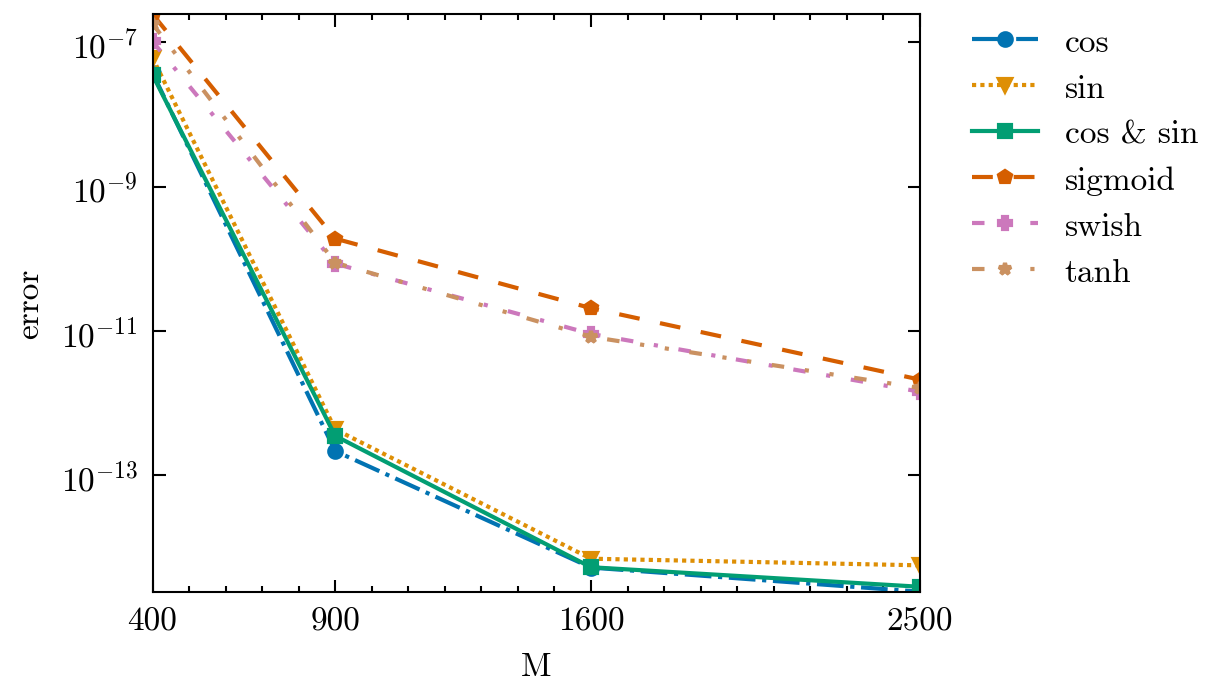}
    \end{minipage}
    \caption{Wave equation: $L_{\infty}$ errors of neural networks when solving the wave equation \eqref{eq:Wave_equation}.}
    \label{fig:error_Wave}
\end{figure}

\begin{table}[htp]
    \begin{center}
        \caption{Wave equation: Performance comparison of FENs and ELMs activated by $\text{sigmoid}$, $\tanh$ and $\text{swish}$ when solving the wave equation \eqref{eq:Wave_equation}. The $L_{\infty}$ errors and $L_{2}$ errors for each model configuration are presented.}
        \setlength\tabcolsep{2pt}
        \small{
        \begin{tabular}{ccccccccc}
            \hline\noalign{\smallskip}
            \multirow{2}{*}{Activations} & \multicolumn{2}{c}{M=400} & \multicolumn{2}{c}{M=900} & \multicolumn{2}{c}{M=1600} & \multicolumn{2}{c}{M=2500}  \\
            & $e_{L_{\infty}}$ & $e_{L_{2}}$ & $e_{L_{\infty}}$ & $e_{L_{2}}$ & $e_{L_{\infty}}$ & $e_{L_{2}}$ & $e_{L_{\infty}}$ & $e_{L_{2}}$     \\
            \hline
            $\text{sigmoid}$      & 2.4436E-07 & 6.4600E-08 & 1.9004E-10 & 1.3386E-10 & 2.0691E-11 & 7.1025E-12 & 2.1174E-12 & 1.0359E-12  \\
            $\tanh$         & 1.7975E-07 & 3.5759E-08 & 8.8139E-11 & 4.3830E-11 & 8.2991E-12 & 3.5998E-12 & 1.6200E-12 & 5.1755E-13  \\
            $\text{swish}$        & 1.0241E-07 & 3.8084E-08 & 8.5493E-11 & 3.3717E-11 & 9.0949E-12 & 4.1869E-12 & 1.4412E-12 & 6.2451E-13  \\
            $\cos$          & 3.7392E-08 & 1.6164E-08 & 2.1672E-13 & 1.2717E-13 & 5.2736E-15 & 2.3601E-15 & 2.4425E-15 & 1.7659E-15  \\
            $\sin$          & 6.0666E-08 & 2.4112E-08 & 4.3354E-13 & 2.1115E-13 & 6.9944E-15 & 3.3327E-15 & 5.6621E-15 & 3.9232E-15  \\
            $\cos$ $\&$ $\sin$ & 3.5241E-08 & 1.0682E-08 & 3.5172E-13 & 1.5920E-13 & 5.3291E-15 & 2.0016E-15 & 2.8588E-15 & 2.6000E-15  \\
            \hline
        \end{tabular}
        }
        \label{tab:error_Wave}
    \end{center}
\end{table}

\subsection{Nonlinear Helmholtz equation}
For the nonlinear example, we evaluate the performance of FENs and ELMs on a boundary value problem governed by the one-dimensional nonlinear Helmholtz equation. The formulation is defined as follows
\begin{equation}
    \label{eq:nonlinear_Helmholtz_equation}
    \begin{array}{r@{}l}
        \left\{
        \begin{aligned}
            \frac{\partial ^2 u}{\partial x^2} -\lambda u + \beta \sin(u) & = f(x), &  & x \in (a, b),          \\
            u(a)             & = h_1,       \\
            u(b)             & = h_2.   \\
        \end{aligned}
        \right.
    \end{array}
\end{equation}
The exact solution is given by
$u(x)=\sin\left(3\pi x + \frac{3\pi}{20}\right) \cos\left(4\pi x - \frac{2\pi}{5}\right) + \frac{3}{2} + \frac{x}{10}$,
with suitable boundary conditions and source term $f(x)$. 
The constant parameters in the equation are set as $a=0$, $b=8$, $\lambda=50$ and $\beta=10$.

In this nonlinear case, we use $N_x = 3000$ uniform collocation points for training. To handle the nonlinearity of the problem, we adopt the Picard iteration method, performing a total of $100$ iterations. Initially, a vector of coefficients $\boldsymbol{w}_0$ is randomly initialized, generating‌ the approximate solution $u_0 = \Phi \cdot \boldsymbol{w}_0$. The nonlinear term $\beta \sin(u_0)$ is then computed  and moved to the right-hand side of the equation, thereby linearizing the problem ‌for the current iteration step‌. Solving the resulting linear system yields an updated coefficient vector $\boldsymbol{w}_1$. This iterative process continues, updating the coefficients at each step until convergence is achieved. The final approximate solution is expressed as $u = \Phi \cdot \boldsymbol{w}$, where $\boldsymbol{w}$ is the ‌converged‌ coefficient vector.

For the nonlinear Helmholtz problem, the parameters ‌along with‌ the corresponding optimal scaling factor ‌identified via scale search‌ are ‌listed‌ in Table~\ref{tab:params_Helmholtz1D}. The resulting $L_{\infty}$ and $L_2$ errors from FENs and ELMs using the optimal scaling factor‌ are ‌presented‌ in Table~\ref{tab:error_Helmholtz1D}. Figure~\ref{fig:error_Helmholtz1D} displays the $L_{\infty}$ error curves for both methods. As clearly shown in the figure‌, the error curves ‌for‌ FENs are consistently and significantly lower than ‌those for‌ ELMs, ‌demonstrating‌ the superior accuracy and robustness of FENs in solving this nonlinear problem.

For the nonlinear Helmholtz problem ‌studied‌, the ELM with the $\tanh$ activation function ‌achieves‌ clearly ‌superior‌ performance compared to ‌networks using‌ $\text{sigmoid}$ and $\text{swish}$, ‌delivering‌ notably higher accuracy ‌while still lagging behind FENs‌. 
This ‌demonstrates‌ that the choice of activation function ‌plays a critical role‌ in the efficacy of ELMs, with $\tanh$ ‌providing‌ a distinct advantage ‌in this specific context‌. ‌However‌, FENs ‌maintain‌ superior representational ‌capacity‌ and precision, ‌highlighting‌ their robustness in addressing ‌this class of‌ nonlinear problems.

Upon examining the approximation error values in Table~\ref{tab:error_Helmholtz1D}, we observe that the smallest $L_{\infty}$ and $L_2$ errors achieved by FENs are $2.2427 \times 10^{-13}$ and $1.3705 \times 10^{-15}$, respectively. In contrast, the lowest $L_{\infty}$ and $L_2$ errors obtained by ELMs are $3.7986 \times 10^{-11}$ and $4.4912 \times 10^{-12}$, respectively. These results clearly demonstrate that FENs are more suitable for solving this nonlinear problem and are capable of achieving significantly higher solution accuracy compared to ELMs.

\begin{table}[htp]
    \begin{center}
        \caption{Nonlinear Helmholtz equation: Parameters when solving the nonlinear Helmholtz equation \eqref{eq:nonlinear_Helmholtz_equation}.}
        \setlength\tabcolsep{2pt}
        \small{
        \begin{tabular}{ccccccc}
            \hline\noalign{\smallskip}
            \multirow{2}{*}{Activations} & \multirow{2}{*}{$(\rho_{min}, \rho_{max}]$} & \multirow{2}{*}{$\rho_s$} & \multicolumn{4}{c}{$\rho_{opt}$}  \\
            & &          & $M=400$ & $M=900$ & $M=1600$ & $M=2500$ \\
            \hline
            $\text{sigmoid}$      & $(0, 10]$  & 0.01 & 9.65 & 9.99 & 9.95 & 9.72  \\
            $\tanh$         & $(0, 10]$  & 0.01 & 9.65 & 9.71 & 9.9  & 9.98  \\
            $\text{swish}$        & $(0, 10]$  & 0.01 & 10.0 & 9.94 & 9.94 & 9.84  \\
            $\cos$          & $(0, 100]$ & 1    & 16   & 33   & 34   & 64  \\
            $\sin$          & $(0, 100]$ & 1    & 20   & 34   & 24   & 74  \\
            $\cos$ $\&$ $\sin$ & $(0, 100]$ & 1    & 22   & 25   & 33   & 71  \\
            \hline
        \end{tabular}
        }
        \label{tab:params_Helmholtz1D}
    \end{center}
\end{table}

\begin{figure}[htbp]
    \begin{minipage}{0.9\linewidth}
        \centering
        \includegraphics[width=0.8\textwidth]{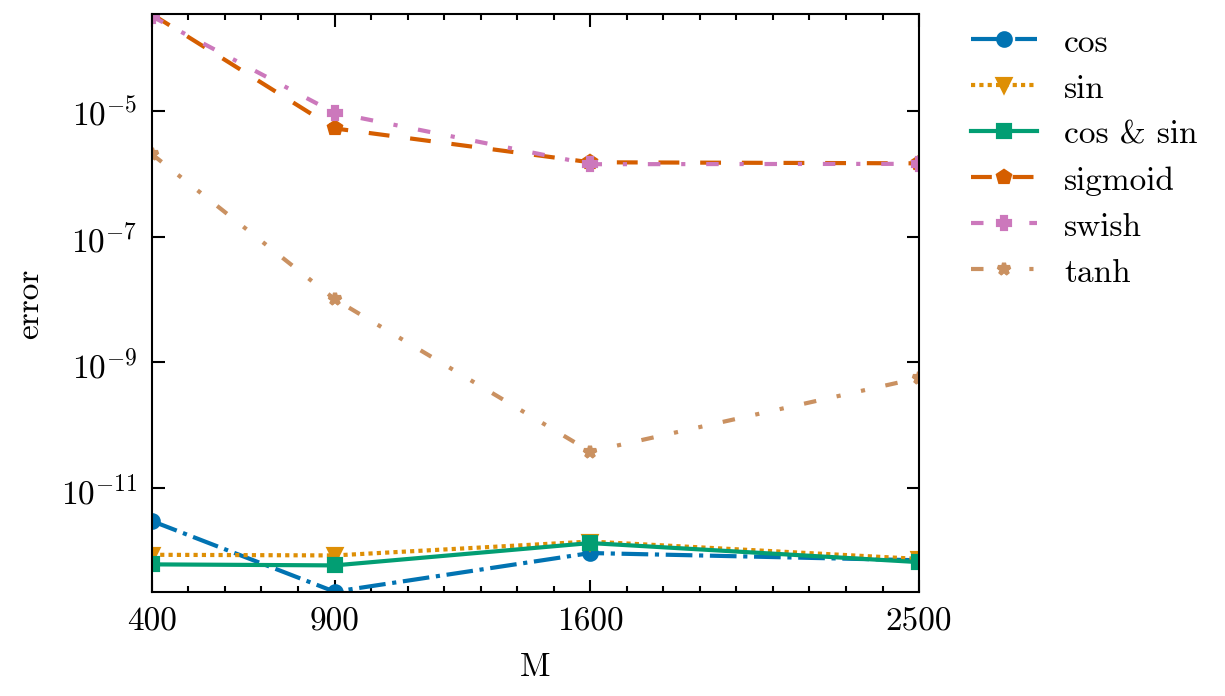}
    \end{minipage}
    \caption{Nonlinear Helmholtz equation: $L_{\infty}$ errors of neural networks when solving the nonlinear Helmholtz equation \eqref{eq:nonlinear_Helmholtz_equation}.}
    \label{fig:error_Helmholtz1D}
\end{figure}

\begin{table}[htp]
    \begin{center}
        \caption{Nonlinear Helmholtz equation: Performance comparison of FENs and ELMs activated by $\text{sigmoid}$, $\tanh$ and $\text{swish}$ when solving the nonlinear Helmholtz equation \eqref{eq:nonlinear_Helmholtz_equation}. The $L_{\infty}$ errors and $L_{2}$ errors for each model configuration are presented.}
        \setlength\tabcolsep{2pt}
        \small{
        \begin{tabular}{ccccccccc}
            \hline\noalign{\smallskip}
            \multirow{2}{*}{Activations} & \multicolumn{2}{c}{M=400} & \multicolumn{2}{c}{M=900} & \multicolumn{2}{c}{M=1600} & \multicolumn{2}{c}{M=2500}  \\
            & $e_{L_{\infty}}$ & $e_{L_{2}}$ & $e_{L_{\infty}}$ & $e_{L_{2}}$ & $e_{L_{\infty}}$ & $e_{L_{2}}$ & $e_{L_{\infty}}$ & $e_{L_{2}}$     \\
            \hline
            $\text{sigmoid}$      & 3.4384E-04 & 1.6943E-05 & 5.2436E-06 & 2.4732E-07 & 1.5165E-06 & 7.1510E-08 & 1.4590E-06 & 7.7030E-08  \\
            $\tanh$         & 2.0442E-06 & 1.0492E-07 & 1.0126E-08 & 4.7025E-10 & 3.7986E-11 & 4.4912E-12 & 5.6331E-10 & 3.1179E-11  \\
            $\text{swish}$        & 3.1724E-04 & 1.9411E-05 & 9.2887E-06 & 5.5102E-07 & 1.4125E-06 & 6.6785E-08 & 1.4330E-06 & 8.1166E-08  \\
            $\cos$          & 2.9874E-12 & 1.5546E-13 & 2.2427E-13 & 1.3705E-14 & 9.2371E-13 & 4.8289E-14 & 7.0832E-13 & 3.5773E-14  \\
            $\sin$          & 8.6642E-13 & 4.2264E-14 & 8.4621E-13 & 4.6449E-14 & 1.4064E-12 & 6.4684E-14 & 7.4474E-13 & 3.7450E-14  \\
            $\cos$ $\&$ $\sin$ & 6.0973E-13 & 2.8716E-14 & 5.8753E-13 & 3.4123E-14 & 1.3309E-12 & 6.4362E-14 & 6.6214E-13 & 3.1905E-14  \\
            \hline
        \end{tabular}
        }
        \label{tab:error_Helmholtz1D}
    \end{center}
\end{table}

\subsection{Poisson equation with an oscillating solution}
We consider the following one-dimensional Poisson equation:
\begin{equation}
    \label{eq:Poisson_equation_sharp}
    \begin{array}{r@{}l}
        \left\{
        \begin{aligned}
            -\Delta u(x) & = f(x), &  & \mbox{in} \enspace \Omega,          \\
            u(x)             & = g(x),        &  & \mbox{on} \enspace \partial \Omega,
        \end{aligned}
        \right.
    \end{array}
\end{equation}
where $\Omega=(0, 1)$.  We choose the suitable $f(x)$ and boundary condition $g(x)$ so that the exact solution is given by
\begin{equation}
    \label{eq:Poisson_equation_oscillating_solution}
    u(x) = \frac{1}{6} \sum_{i=1}^{6} \sin (2^i \pi x),
\end{equation}
which represents a superposition of sine functions with exponentially increasing frequencies, producing a highly oscillatory ‌behavior. 

We train the networks using $N_x = 3000$ uniformly distributed collocation points. The parameters associated with different neural architectures‌ during the optimal scaling factor search across different numbers of basis functions, are summarized in Table~\ref{tab:error_Poisson1D}. Notably, the admissible range for the activation function $\tanh$ is narrower than those of $\text{sigmoid}$ and $\text{swish}$. This constraint ‌originates from‌ numerical stability requirements: ‌applying‌ excessively large scaling factors to $\tanh$ ‌generates‌ ‌severely‌ ill-conditioned matrices (‌approaching singularity‌), ‌which destabilizes the solving process. ‌To mitigate this, the $\tanh$ scaling range is ‌strategically restricted‌ to maintain ‌trainable system ‌conditions.

Table~\ref{tab:params_Poisson1D} compiles the optimal scaling factors ‌and their associated search parameters. Furthermore, Figure~\ref{fig:Poisson1D} shows a comparison between the exact solution and the numerical solution obtained using the FEN with the $\cos$ activation function, where the number of basis functions is set to $M = 900$ and the optimal scaling factor is $\rho_{\text{opt}} = 130$. It is observed that the two curves exhibit an almost perfect overlap, demonstrating that the neural network has successfully captured the exact solution with high accuracy.

Figure~\ref{fig:error_Poisson1D} quantifies the $L_{\infty}$ 
convergence of neural networks ‌equipped with‌ optimal scaling factors for solving the ‌highly oscillatory‌ Poisson equation~\eqref{eq:Poisson_equation_oscillating_solution}. It is clearly observed that the errors of ELMs with $\text{sigmoid}$, $\text{swish}$, and $\tanh$ activation functions are significantly higher than those of FENs. In Table~\ref{tab:error_Poisson1D}, we report the $L_{\infty}$ and $L_2$ errors corresponding to various numbers of basis functions. The minimum $L_{\infty}$ and $L_2$ errors achieved by ELMs are $6.1572 \times 10^{-8}$ and $8.0396 \times 10^{-8}$, respectively, whereas FENs achieve notably smaller errors of $1.3878 \times 10^{-11}$ and $3.0395 \times 10^{-11}$. These results highlight the superior representational power of FENs and demonstrate their greater suitability for solving this highly oscillatory Poisson problem.

\begin{table}[htp]
    \begin{center}
        \caption{Poisson equation with an oscillating solution: Parameters when solving the Poisson equation \eqref{eq:Poisson_equation_sharp}.}
        \setlength\tabcolsep{2pt}
        \small{
        \begin{tabular}{ccccccc}
            \hline\noalign{\smallskip}
            \multirow{2}{*}{Activations} & \multirow{2}{*}{$(\rho_{min}, \rho_{max}]$} & \multirow{2}{*}{$\rho_s$} & \multicolumn{4}{c}{$\rho_{opt}$}  \\
            & &          & $M=400$ & $M=900$ & $M=1600$ & $M=2500$ \\
            \hline
            $\text{sigmoid}$      & $(0, 20]$  & 0.01 & 18.79 & 18.12 & 18.46 & 19.19  \\
            $\tanh$         & $(0, 10]$  & 0.01 & 9.15 & 9.83 & 9.93 & 9.70  \\
            $\text{swish}$        & $(0, 20]$  & 0.01 & 19.90 & 19.33 & 19.39 & 18.95  \\
            $\cos$          & $(0, 1000]$ & 1    & 176 & 130 & 205 & 564  \\
            $\sin$          & $(0, 1000]$ & 1    & 147 & 150 & 234 & 890  \\
            $\cos$ $\&$ $\sin$ & $(0, 1000]$ & 1    & 164 & 146 & 325 & 903  \\
            \hline
        \end{tabular}
        }
        \label{tab:params_Poisson1D}
    \end{center}
\end{table}

\begin{figure}[htbp]
    \begin{minipage}{0.9\linewidth}
        \centering
        \includegraphics[width=0.8\textwidth]{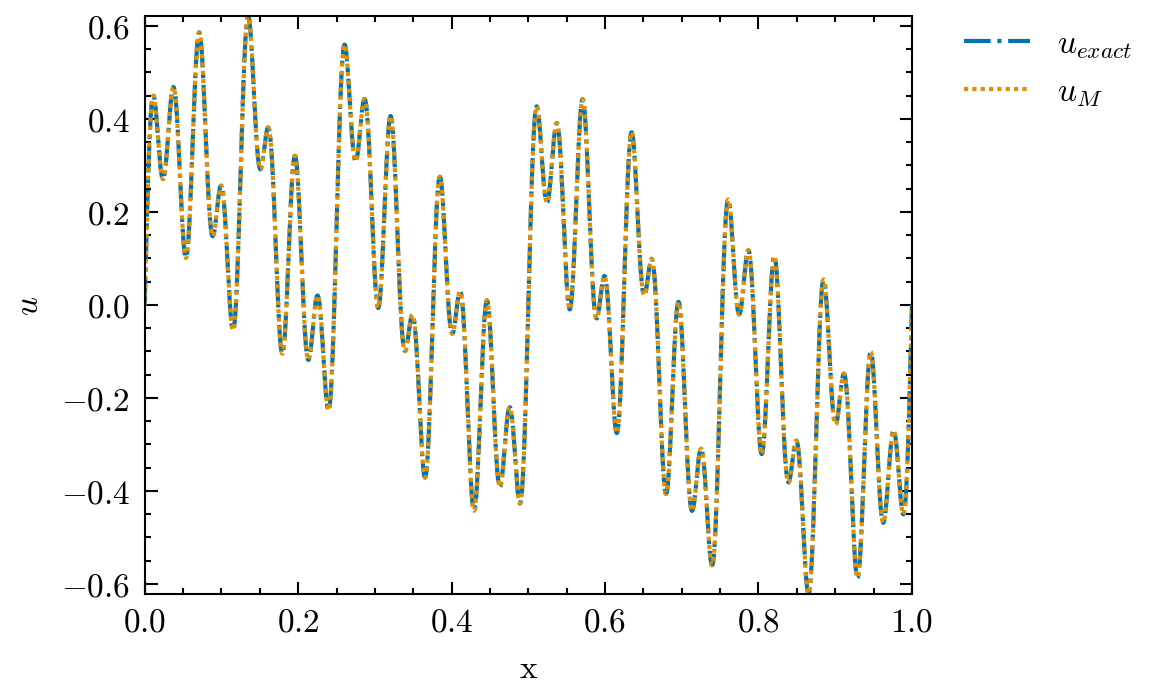}
    \end{minipage}
    \caption{Poisson equation with an oscillating solution: comparison between the exact solution and the approximate solution obtained using FEN with a $\cos$ activation, where $M=900$ and $\rho_{opt}=130$.}
    \label{fig:Poisson1D}
\end{figure}

\begin{figure}[htbp]
    \begin{minipage}{0.9\linewidth}
        \centering
        \includegraphics[width=0.8\textwidth]{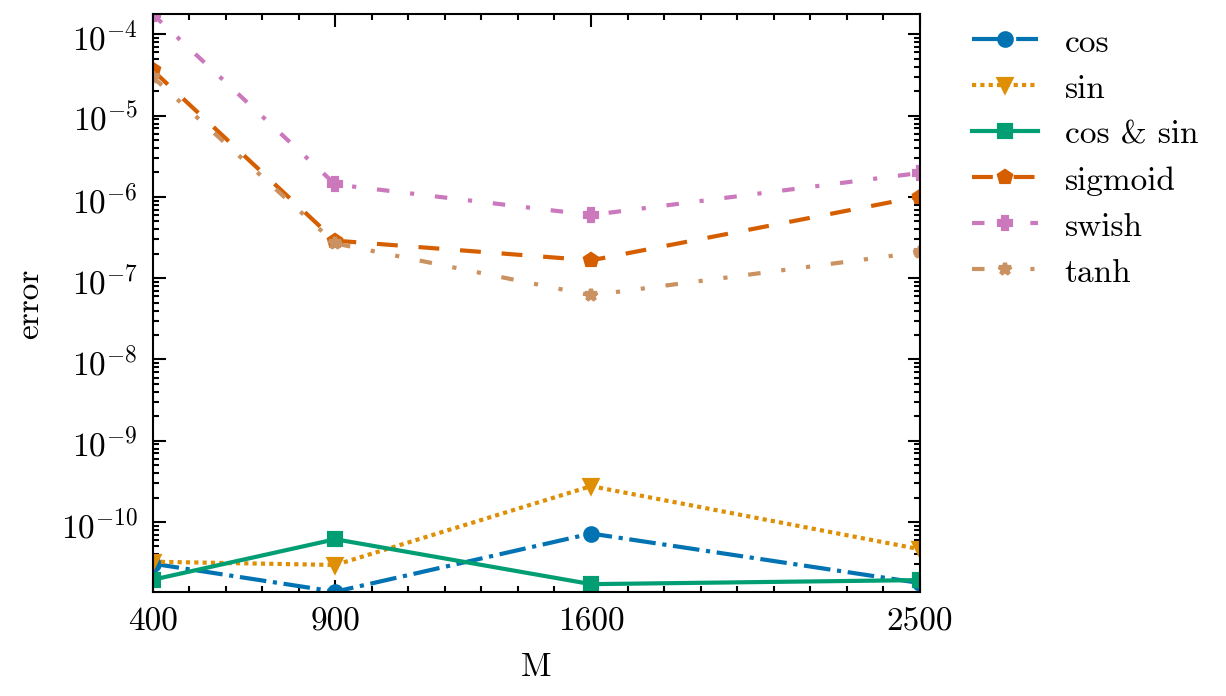}
    \end{minipage}
    \caption{Poisson equation with an oscillating solution: $L_{\infty}$ errors of neural networks when solving the Poisson equation \eqref{eq:Poisson_equation_sharp}.}
    \label{fig:error_Poisson1D}
\end{figure}

\begin{table}[htp]
    \begin{center}
        \caption{Poisson equation with an oscillating solution: Performance comparison of FENs and ELMs activated by $\text{sigmoid}$, $\tanh$ and $\text{swish}$ when solving the Poisson equation \eqref{eq:Poisson_equation_sharp}. The $L_{\infty}$ errors and $L_{2}$ errors for each model configuration are presented.}
        \setlength\tabcolsep{2pt}
        \small{
        \begin{tabular}{ccccccccc}
            \hline\noalign{\smallskip}
            \multirow{2}{*}{Activations} & \multicolumn{2}{c}{M=400} & \multicolumn{2}{c}{M=900} & \multicolumn{2}{c}{M=1600} & \multicolumn{2}{c}{M=2500}  \\
            & $e_{L_{\infty}}$ & $e_{L_{2}}$ & $e_{L_{\infty}}$ & $e_{L_{2}}$ & $e_{L_{\infty}}$ & $e_{L_{2}}$ & $e_{L_{\infty}}$ & $e_{L_{2}}$     \\
            \hline
            $\text{sigmoid}$      & 3.6288E-05 & 3.4848E-05 & 2.9000E-07 & 6.0888E-07 & 1.6587E-07 & 1.8036E-07 & 1.0117E-06 & 1.0275E-06  \\
            $\tanh$         & 3.0062E-05 & 3.8163E-07 & 2.6801E-07 & 2.8331E-07 & 6.1572E-08 & 8.0396E-08 & 2.0924E-07 & 1.7937E-07  \\
            $\text{swish}$        & 1.7748E-04 & 2.4055E-04 & 1.4265E-06 & 1.5708E-06 & 5.9915E-07 & 8.1065E-07 & 1.9759E-06 & 2.1461E-06  \\
            $\cos$          & 3.0724E-11 & 9.7238E-11 & 1.3878E-11 & 3.0395E-11 & 7.2102E-11 & 9.4554E-11 & 1.7749E-11 & 3.2850E-11  \\
            $\sin$          & 3.2458E-11 & 5.7302E-11 & 2.9600E-11 & 4.9456E-11 & 2.7607E-10 & 4.0441E-10 & 4.6233E-11 & 9.1841E-11  \\
            $\cos$ $\&$ $\sin$ & 1.9459E-11 & 3.0421E-11 & 6.1339E-11 & 1.1078E-10 & 1.7186E-11 & 2.8037E-11 & 1.9275E-11 & 3.1528E-11  \\
            \hline
        \end{tabular}
        }
        \label{tab:error_Poisson1D}
    \end{center}
\end{table}

\subsection{Nonlinear Burgers' equation} 
Consider the following 1D Burgers' equation
\begin{equation}
    \label{eq:Nonlinear_Burgers}
    \begin{array}{r@{}l}
        \left\{
        \begin{aligned}
            u_t + u u_x - \epsilon u_{xx} & = f, &  & \mbox{in} \enspace \Omega \times (0, t_f], \\
            u             & = g,        &  & \mbox{on} \enspace \partial \Omega \times (0, t_f], \\
            u & = u_0,  &  & \mbox{in} \enspace \Omega,
        \end{aligned}
        \right.
    \end{array}
\end{equation}
where 
$\Omega=(0, 1)$ and $t_f = 1$. We choose the suitable $f(x,t)$, $g(x,t)$ and $u_0(x)$ so that the exact solution is given by
\begin{equation}
    \label{eq:Nonlinear_Burgers_solution}
    u(t, x) = \frac{1}{1+e^{\frac{x-t}{2\epsilon}}},
\end{equation}
with a small value of $\epsilon=0.01$.

We use $N_x \times N_t = 200 \times 200$ uniform collocation points to train the neural network. Similarly to the one-dimensional nonlinear Helmholtz equation, we employ Picard iteration to solve this problem, implementing $100$ iterations. In Table \ref{tab:params_Burgers}, we show the optimal scaling factors and the relevant parameters used during the search process. In Figure \ref{fig:Burgers2D}, we present heat maps of the exact solution, the approximate solution, and the absolute error between them. The approximate solution is obtained using FEN with a $\sin$ activation, where the number of basis functions is $M = 5000$ and the optimal scaling factor is $\rho_{opt} = 112$. It can be observed that, while the approximate solution captures the details well, the error is noticeably larger compared to other problems due to the solution has large gradient in a local area.

In Figure \ref{fig:error_Burgers}, we present the $L_{\infty}$ error curves. It can be observed that when the number of basis functions $M \leq 2500$, none of the networks approximate the exact solution well. However, when $M = 5000$, FENs outperform ELMs in approximating the exact solution. Although the representational ability of ELMs has improved, the error achieved by FENs is still much lower than that of ELMs, indicating that ELMs are not suitable for solving this problem. The $L_{\infty}$ and $L_2$ errors listed in Table \ref{tab:error_Burgers} further support this conclusion. \chadded{In addition, it can be observed that both FENs and ELMs exhibit a noticeable decrease in accuracy when solving this nonlinear Burgers' equation, which is attributed to the presence of sharp variations in the solution.}

\begin{table}[htp]
    \begin{center}
        \caption{Nonlinear Burgers' equation: Parameters when solving the nonlinear Burgers' equation \eqref{eq:Nonlinear_Burgers}.}
        \setlength\tabcolsep{2pt}
        \small{
        \begin{tabular}{ccccccc}
            \hline\noalign{\smallskip}
            \multirow{2}{*}{Activations} & \multirow{2}{*}{$(\rho_{min}, \rho_{max}]$} & \multirow{2}{*}{$\rho_s$} & \multicolumn{4}{c}{$\rho_{opt}$}  \\
            & &          & $M=900$ & $M=1600$ & $M=2500$ & $M=5000$ \\
            \hline
            $\text{sigmoid}$      & $(0, 10]$  & 0.1 & 6.7 & 8.1 & 5.0 & 9.7  \\
            $\tanh$         & $(0, 10]$  & 0.1 & 6.3 & 5.7 & 3.5 & 5.9  \\
            $\text{swish}$        & $(0, 10]$  & 0.1 & 9.9 & 5.7 & 9.0 & 9.6  \\
            $\cos$          & $(0, 150]$ & 1    & 45 & 62 & 79 & 111  \\
            $\sin$          & $(0, 150]$ & 1    & 44 & 49 & 78 & 112  \\
            $\cos$ $\&$ $\sin$ & $(0, 150]$ & 1    & 50 & 62 & 80 & 114  \\
            \hline
        \end{tabular}
        }
        \label{tab:params_Burgers}
    \end{center}
\end{table}

\begin{figure}[htbp]
    \begin{minipage}{0.9\linewidth}
        \centering
        \includegraphics[width=1\textwidth]{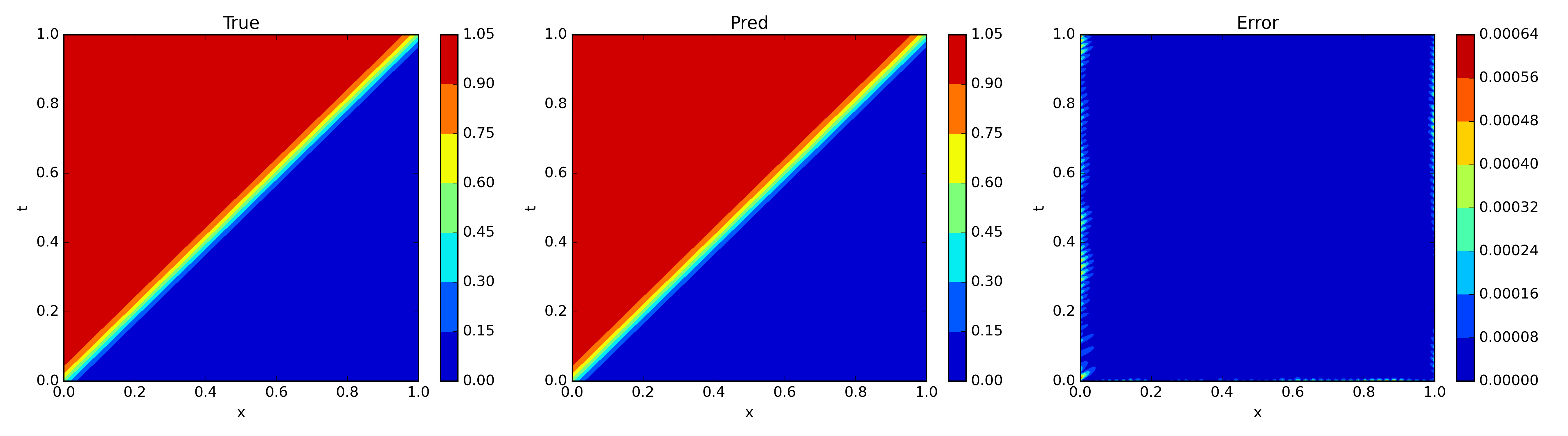}
    \end{minipage}
    \caption{Nonlinear Burgers' equation: the heat maps of the exact solution and the approximate solution of FEN with a $\sin$ activation, where $M=5000$ and $\rho_{opt}=112$.}
    \label{fig:Burgers2D}
\end{figure}

\begin{figure}[htbp]
    \begin{minipage}{0.9\linewidth}
        \centering
        \includegraphics[width=0.8\textwidth]{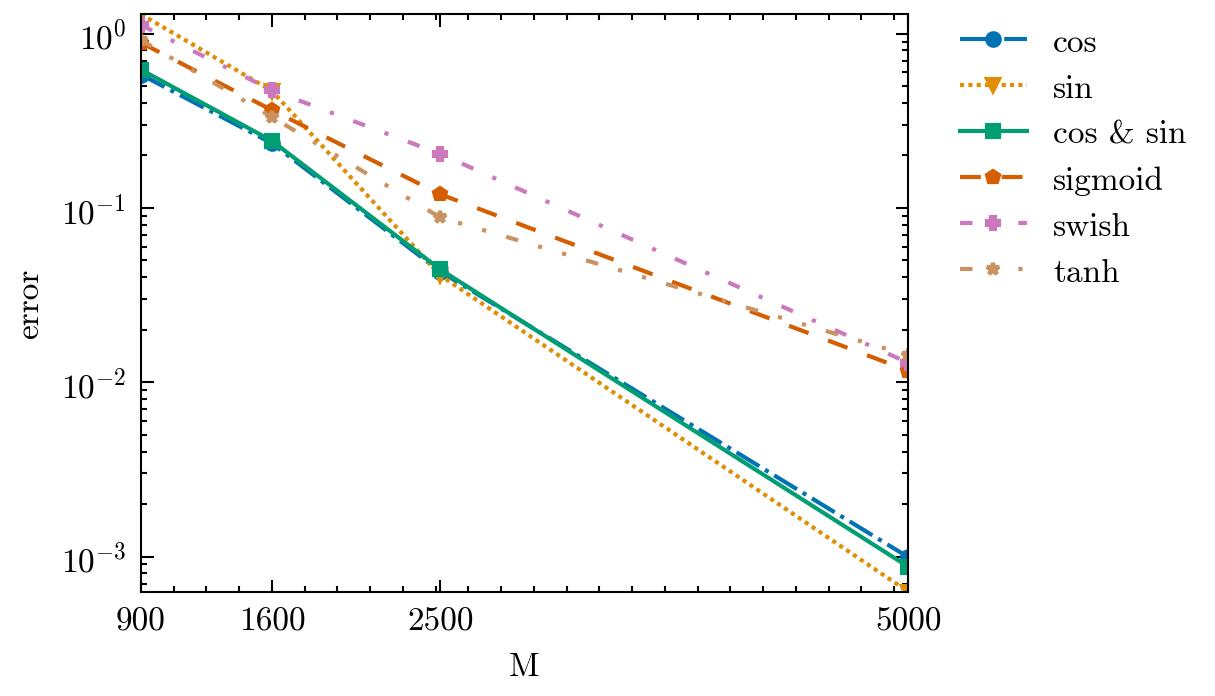}
    \end{minipage}
    \caption{Nonlinear Burgers' equation: $L_{\infty}$ errors of neural networks when solving the nonlinear Burgers' equation \eqref{eq:Nonlinear_Burgers}.}
    \label{fig:error_Burgers}
\end{figure}

\begin{table}[htp]
    \begin{center}
        \caption{Nonlinear Burgers' equation: Performance comparison of FENs and ELMs activated by $\text{sigmoid}$, $\tanh$ and $\text{swish}$ when solving the nonlinear Burgers' equation \eqref{eq:Nonlinear_Burgers}. The $L_{\infty}$ errors and $L_{2}$ errors for each model configuration are presented.}
        \setlength\tabcolsep{2pt}
        \small{
        \begin{tabular}{ccccccccc}
            \hline\noalign{\smallskip}
            \multirow{2}{*}{Activations} & \multicolumn{2}{c}{M=900} & \multicolumn{2}{c}{M=1600} & \multicolumn{2}{c}{M=2500} & \multicolumn{2}{c}{M=5000}  \\
            & $e_{L_{\infty}}$ & $e_{L_{2}}$ & $e_{L_{\infty}}$ & $e_{L_{2}}$ & $e_{L_{\infty}}$ & $e_{L_{2}}$ & $e_{L_{\infty}}$ & $e_{L_{2}}$     \\
            \hline
            $\text{sigmoid}$       & 8.8672E-01 & 2.0842E-01 & 3.6381E-01 & 1.3008E-01 & 1.1966E-01 & 1.6866E-02 & 1.1641E-02 & 1.1815E-03 \\
            $\tanh$          & 9.1862E-01 & 2.0184E-01 & 3.3211E-01 & 1.1212E-01 & 8.8108E-02 & 2.0741E-02 & 1.4296E-02 & 1.9656E-03 \\
            $\text{swish}$         & 1.1182E+00 & 6.4755E-01 & 4.7318E-01 & 9.8497E-02 & 2.0387E-01 & 2.9728E-02 & 1.2937E-02 & 2.9260E-03 \\
            $\cos$           & 5.8175E-01 & 1.4085E-01 & 2.3469E-01 & 5.9826E-02 & 4.3138E-02 & 5.1853E-03 & 9.9221E-04 & 4.8742E-05 \\
            $\sin$           & 1.2913E+00 & 3.7190E-01 & 4.6918E-01 & 1.3188E-01 & 4.1016E-02 & 5.6931E-03 & 6.2835E-04 & 5.4839E-05 \\
            $\cos$ $\&$ $\sin$  & 6.2089E-01 & 3.0694E-01 & 2.4240E-01 & 6.6556E-02 & 4.4349E-02 & 5.9443E-03 & 8.7493E-04 & 6.0509E-05 \\
            \hline
        \end{tabular}
        }
        \label{tab:error_Burgers}
    \end{center}
\end{table}

\subsection{High-dimensional Poisson equation}
We consider the high-dimensional Poisson equation given by Equation~\eqref{eq:HDPoisson_equation}
\begin{equation}
    \label{eq:HDPoisson_equation}
    \begin{array}{r@{}l}
        \left\{
        \begin{aligned}
            -\Delta u & = f(\boldsymbol{x}), &  & \mbox{in} \enspace \Omega,          \\
            u             & = h(\boldsymbol{x}),        &  & \mbox{on} \enspace \partial \Omega,
        \end{aligned}
        \right.
    \end{array}
\end{equation}
where $\Omega = (-1, 1)^d$ represents the spatial domain.
The exact solution to this equation is provided by 
\begin{equation}
    \label{eq:HDPoisson_solution}
    u(\boldsymbol{x}) = \left(\frac{1}{d}\sum_{i=1}^d x_i\right)^2 + \sin\left(\frac{1}{d}\sum_{i=1}^d x_i\right),
\end{equation}
with suitable $f(\boldsymbol{x})$ and boundary condition $h(\boldsymbol{x})$.

For the training of both FENs and ELMs, we randomly select $50{,}000$ collocation points in the domain $\Omega$ and $1000d$ collocation points on the boundary $\partial \Omega$. Unlike in previous examples, in this problem we fix the number of basis functions to $M = 10{,}000$. This is because, for high-dimensional problems, a sufficient number of basis functions is essential to capture the complexity of the solution space. To evaluate the capability of our methods in solving high-dimensional problems, we solve the Poisson equation in dimensions $5$, $7$, $10$, and $15$.

In Table~\ref{tab:params_PoissonHD}, we provide the parameters used for the optimal scale search, along with the optimal scaling factors obtained for problems of various dimensions. Table~\ref{tab:error_PoissonHD} summarizes the $L_{\infty}$ and $L_2$ errors achieved by the FENs and ELMs when using the optimal scaling factors. Figure~\ref{fig:error_PoissonHD} displays the $L_{\infty}$ error curves.

The error distributions ‌reveal critical comparative information about‌ the dimensional scalability of FENs and ELMs. At $d=5$, the errors of FENs are significantly lower than those of ELMs. However, this trend changes as the dimensionality increases. In higher dimensions, the ELM with the $\text{sigmoid}$ activation and the FEN with the $\sin$ activation exhibit comparable precisions.
The ELM with the $\tanh$ activation shows the highest error, indicating its limited ability to achieve high precision. Although the ELM with the $\text{swish}$ activation performs better than those with $\text{sigmoid}$ and $\tanh$ activations at $d=5$, its performance deteriorates in higher dimensions. In these cases, its errors fall between the two and are comparable to those of FENs with the $\cos$ activation, albeit still slightly higher than those of FENs with combined $\cos$ and $\sin$ activations.

From the perspective of specific approximation error values, at $d=5$, FENs achieve the smallest $L_{\infty}$ and $L_2$ errors of $1.2396 \times 10^{-13}$ and $3.9641 \times 10^{-14}$, respectively. In comparison, ELMs reach minimum $L_{\infty}$ and $L_2$ errors of $1.5994 \times 10^{-12}$ and $4.5378 \times 10^{-13}$, respectively. At higher dimensions ($d=7$, $10$, and $15$), both FENs and ELMs achieve $L_{\infty}$ and $L_2$ errors of similar magnitudes: approximately $10^{-9}$ and $10^{-10}$ for $d=7$, $10^{-6}$ and $10^{-7}$ for $d=10$, and $10^{-5}$ for both metrics at $d=15$. These numerical results highlight the performance of FENs and ELMs across different dimensions. While FENs clearly outperform ELMs at lower dimensions in terms of approximation accuracy, the performance gap narrows as dimensionality increases, with both methods achieving comparable levels of precision.

\begin{table}[htp]
    \begin{center}
        \caption{High-dimensional Poisson equation: Parameters when solving the high dimensional Poisson equation \eqref{eq:HDPoisson_equation}.}
        \setlength\tabcolsep{2pt}
        \small{
        \begin{tabular}{ccccccc}
            \hline\noalign{\smallskip}
            \multirow{2}{*}{Activations} & \multirow{2}{*}{$(\rho_{min}, \rho_{max}]$} & \multirow{2}{*}{$\rho_s$} & \multicolumn{4}{c}{$\rho_{opt}$}  \\
            & &          & $d=5$ & $d=7$ & $d=10$ & $d=15$ \\
            \hline
            $\text{sigmoid}$      & $(0, 1]$ & 0.001 & 0.143 & 0.049 & 0.027 & 0.016  \\
            $\tanh$         & $(0, 1]$ & 0.001 & 0.046 & 0.015 & 0.014 & 0.008  \\
            $\text{swish}$        & $(0, 1]$ & 0.001 & 0.12  & 0.033 & 0.035 & 0.024  \\
            $\cos$          & $(0, 1]$ & 0.001 & 0.25  & 0.111 & 0.091 & 0.043  \\
            $\sin$          & $(0, 1]$ & 0.001 & 0.262 & 0.102 & 0.052 & 0.026  \\
            $\cos$ $\&$ $\sin$ & $(0, 1]$ & 0.001 & 0.265 & 0.105 & 0.069 & 0.05  \\
            \hline
        \end{tabular}
        }
        \label{tab:params_PoissonHD}
    \end{center}
\end{table}

\begin{figure}[htbp]
    \begin{minipage}{0.9\linewidth}
        \centering
        \includegraphics[width=0.8\textwidth]{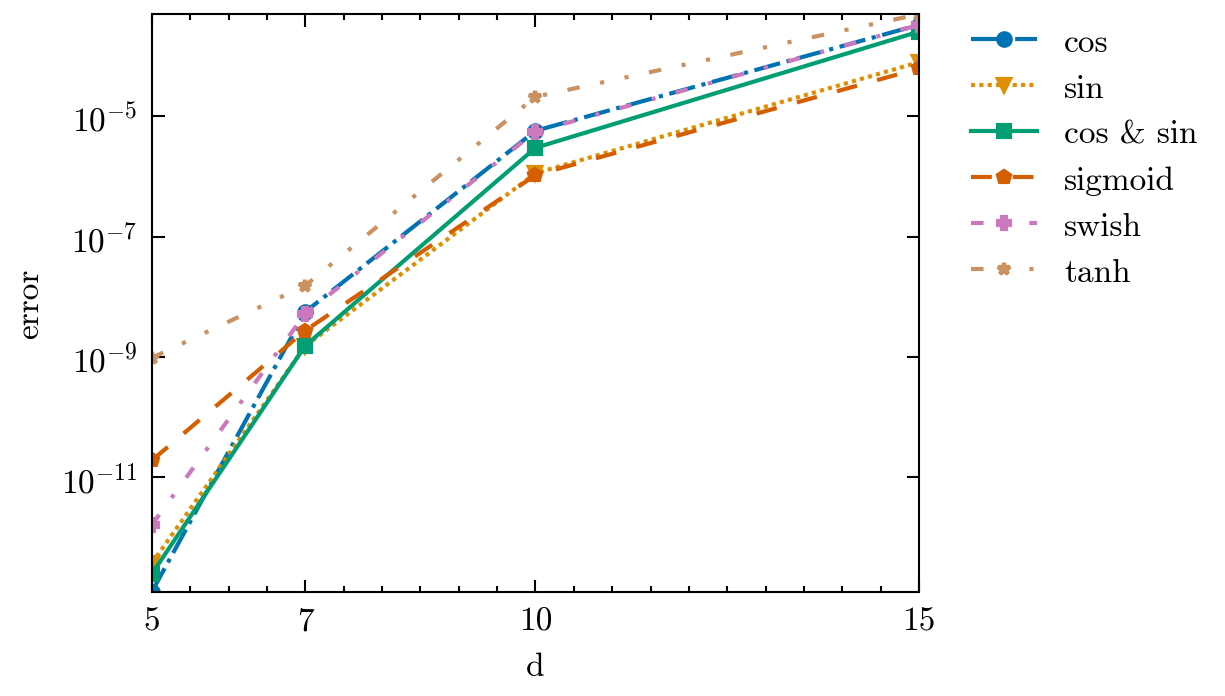}
    \end{minipage}
    \caption{High-dimensional Poisson equation: $L_{\infty}$ errors of neural networks when solving the high dimensional Poisson equation \eqref{eq:HDPoisson_equation}.}
    \label{fig:error_PoissonHD}
\end{figure}

\begin{table}[htp]
    \begin{center}
        \caption{High-dimensional Poisson equation: Performance comparison of FENs and ELMs activated by $\text{sigmoid}$, $\tanh$ and $\text{swish}$ when solving the high dimensional Poisson equation \eqref{eq:nonlinear_Helmholtz_equation}. The $L_{\infty}$ errors and $L_{2}$ errors for each model configuration are presented.}
        \setlength\tabcolsep{2pt}
        \small{
        \begin{tabular}{ccccccccc}
            \hline\noalign{\smallskip}
            \multirow{2}{*}{Activations} & \multicolumn{2}{c}{d=5} & \multicolumn{2}{c}{d=7} & \multicolumn{2}{c}{d=10} & \multicolumn{2}{c}{d=15}  \\
            & $e_{L_{\infty}}$ & $e_{L_{2}}$ & $e_{L_{\infty}}$ & $e_{L_{2}}$ & $e_{L_{\infty}}$ & $e_{L_{2}}$ & $e_{L_{\infty}}$ & $e_{L_{2}}$     \\
            \hline
            $\text{sigmoid}$       & 1.8921E-11 & 8.7116E-12 & 2.7337E-09 & 3.8450E-10 & 1.0511E-06 & 3.1487E-07 & 6.2374E-05 & 4.0279E-05  \\
            $\tanh$          & 9.5997E-10 & 1.0016E-10 & 1.5291E-08 & 1.9115E-09 & 2.1064E-05 & 8.1063E-06 & 4.9823E-04 & 3.1662E-04  \\
            $\text{swish}$         & 1.5994E-12 & 4.5378E-13 & 5.2239E-09 & 1.0120E-09 & 5.5396E-06 & 4.2761E-07 & 3.3137E-04 & 5.9395E-05  \\
            $\cos$           & 1.2396E-13 & 3.9641E-14 & 5.6010E-09 & 7.7473E-10 & 5.7360E-06 & 3.5996E-07 & 3.3140E-04 & 5.9459E-05  \\
            $\sin$           & 3.6643E-13 & 7.9493E-14 & 1.5148E-09 & 3.1943E-10 & 1.1275E-06 & 4.4159E-07 & 7.9356E-05 & 5.2070E-05  \\
            $\cos$ $\&$ $\sin$  & 2.4564E-13 & 5.0408E-14 & 1.5341E-09 & 3.2616E-10 & 2.9906E-06 & 4.9836E-07 & 2.5601E-04 & 3.3998E-05  \\
            \hline
        \end{tabular}
        }
        \label{tab:error_PoissonHD}
    \end{center}
\end{table}

\chadded{
In Table 22, we report the computational time for solving the high-dimensional Poisson equation using FENs and ELMs with $\tanh$, $\text{sigmoid}$, and $\text{swish}$ activation functions. It can be observed that all models exhibit relatively low computational time, primarily because we do not rely on automatic differentiation but instead derive the derivatives of the basis functions analytically. Notably, FENs achieve significantly lower runtime compared to ELMs. This is because the derivatives of the basis functions in FENs with $\cos$/$\sin$ activations are analytically simple, whereas the derivatives of the basis functions in ELMs with $\tanh$, $\text{sigmoid}$, and $\text{swish}$ activations introduce considerable computational complexity. Finally, it should be noted that in the searching for optimal scaling factor, if the range of $\rho$ is too large or the step size $\rho_s$ is too small, the overall computational cost can increase significantly.
}

\begin{table}[htp]
    \begin{center}
        \caption{High-dimensional Poisson equation: Computational time for solving the high dimensional Poisson equation \eqref{eq:HDPoisson_equation}.}
        \setlength\tabcolsep{2pt}
        \small{
        \begin{tabular}{ccccccc}
            \hline\noalign{\smallskip}
            Activations & $d=5$ & $d=7$ & $d=10$ & $d=15$ \\
            \hline
            $\text{sigmoid}$      & 2.3847 & 2.6247 & 3.1373 & 3.7696  \\
            $\tanh$               & 2.2354 & 2.4611 & 2.8802 & 3.5063  \\
            $\text{swish}$        & 2.6865 & 3.1198 & 3.8061 & 4.9203  \\
            $\cos$                & 2.0713 & 2.1767 & 2.5219 & 3.0101  \\
            $\sin$                & 2.0117 & 2.2124 & 2.6537 & 3.1934  \\
            $\cos$ $\&$ $\sin$    & 1.5786 & 1.7345 & 2.0190 & 2.4217  \\
            \hline
        \end{tabular}
        }
        \label{tab:time_PoissonHD}
    \end{center}
\end{table}

\section{Conclusions}
\label{sec:conclusions}
In this work, we propose Fourier Feature Networks (FENs) to study function approximation and the solution of linear and nonlinear PDEs. These networks employ a single-hidden-layer neural network to represent the target function, where the outputs of the hidden layer correspond to a set of basis functions. The linear combination of this set of basis functions yields the representation of the target function. This concept is similar to that of ELMs, which also utilize a single-hidden-layer neural network. However, the key difference is that, in ELMs, affine transformations are indispensable. Without them, high-precision solutions cannot be obtained for certain problems. In contrast, FENs do not require affine transformations. Additionally, while ELMs typically use activation functions such as $\text{sigmoid}$, $\tanh$, and $\text{swish}$, FENs utilize $\cos$, $\sin$, or a combination of both, which naturally introduces Fourier features.

We initialize the input-to-hidden weights and biases by random sampling from a uniform distribution (variance$=1$), which remain fixed during training. Only the output layer's linear combination coefficients require optimization. To fully demonstrate the capability of neural networks in solving problems, the scaling factors for weights and biases are searched within a specified range. The optimal scaling factors, which minimize the error in the algebraic equation, are identified. Problems solved using these optimal scaling factors yield higher precision solutions. In our numerical experiments, both for function approximation and solving linear or nonlinear PDEs, we observe that the solutions obtained by FENs are significantly more accurate than those obtained by ELMs using $\text{sigmoid}$, $\tanh$, or $\text{swish}$ activation functions. In high-dimensional Poisson problems, FENs achieve higher precision solutions even in $5$ dimensions, and for problems with even higher dimensions, FENs can perform comparably to ELMs.

Although the proposed neural networks have shown promising results in numerical experiments, there are still areas that require further research. The neural networks are highly dependent on the choice of scaling factors; poor selection can make it difficult to obtain a high-precision solution. Thus, researching more effective methods for determining the optimal scaling factor is necessary. 
Both FENs and ELMs struggle with high-dimensional problems, so the development of an algorithm capable of handling such problems with greater precision is essential. \chadded{Finally, we consider that analyzing the frequency-domain superiority of trigonometric activation functions over tanh, sigmoid, and swish will constitute an important direction for future work.
}

\section*{Acknowledgment}
This research is partially supported by the National Key R \& D Program of China (No.2022YFE03040002) and the National Natural Science Foundation of China ( No.12371434).

\section*{Data Availability Statement}
The data that support the findings of this study are available from the corresponding author upon reasonable request.

\bibliographystyle{unsrt}
\bibliography{./main}

\end{document}